\documentclass[12pt]{article}

\usepackage{amsmath,amsthm}
\usepackage{amsfonts,amssymb}
\usepackage{stmaryrd}
\usepackage{authordate1-4}

\newtheorem{thm}{Theorem}[section]
\newtheorem{cor}[thm]{Corollary}
\newtheorem{lem}[thm]{Lemma}
\newtheorem{prop}[thm]{Proposition}
\theoremstyle{definition}

\newtheorem{defn}{Definition}[section]

\newcommand  {\llb}		{\llbracket}
\newcommand  {\rrb}		{\rrbracket}
\renewcommand{\empty}	{\varnothing}
\newcommand	 {\CC}		{\mathbf{C}}
\newcommand	 {\II}		{\mathbf{I}}

\newcommand	 {\FF}		{\mathbf{F}}
\newcommand	 {\TT}		{\mathbf{T}}
\newcommand	 {\ee}		{\mathbf{e}}
\newcommand	 {\Assert}	{\mathbf{Assert}}
\newcommand	 {\Refute}	{\mathbf{Refute}}
\newcommand	 {\Test}	{\mathbf{Test}}
\newcommand	 {\Empty}	{\mathbf{Empty}}
\newcommand	 {\Setref}	{\mathbf{Setref}}
\newcommand	 {\Unset}	{\mathbf{Unset}}
\newcommand	 {\Deref}	{\mathbf{Deref}}
\newcommand	 {\Noun}	{\mathbf{Noun}}
\newcommand	 {\Name}	{\mathbf{Name}}
\newcommand	 {\Verbt}	{\mathbf{Verbt}}
\newcommand	 {\VerbBe}	{\mathbf{VerbBe}}
\newcommand	 {\Det}		{\mathbf{Det}}

\newcommand	 {\NP}		{\mathbf{NP}}
\newcommand	 {\NPa}		{\mathbf{NPa}}
\newcommand	 {\NPi}		{\mathbf{NPi}}
\newcommand	 {\VPs}		{\mathbf{VPs}}
\newcommand	 {\VPm}		{\mathbf{VPm}}
\newcommand	 {\VP}		{\mathbf{VP}}
\newcommand	 {\VPa}		{\mathbf{VPa}}
\newcommand	 {\VPi}		{\mathbf{VPi}}
\newcommand	 {\Cl}		{\mathbf{Cl}}
\newcommand	 {\Cli}		{\mathbf{ICl}}
\newcommand	 {\Text}	{\mathbf{Text}}
\newcommand	 {\St}		{\mathbf{S}}
\newcommand	 {\Sd}		{\mathbf{DS}}
\newcommand	 {\Si}		{\mathbf{IS}}
\newcommand	 {\Jack}	{\mathbf{Jack}}
\newcommand	 {\Build}	{\mathbf{Build}}
\newcommand	 {\Sell}	{\mathbf{Sell}}
\newcommand	 {\House}	{\mathbf{House}}
\newcommand	 {\Computer}	{\mathbf{Computer}}
\newcommand	 {\Car}		{\mathbf{Car}}
\newcommand	 {\Builder}	{\mathbf{Builder}}
\newcommand	 {\SNP}		{\mathbf{SNP}}
\newcommand	 {\ONP}		{\mathbf{ONP}}
\newcommand	 {\SPron}	{\mathbf{SPron}}
\newcommand	 {\OPron}	{\mathbf{OPron}}
\newcommand	 {\He}		{\mathbf{He}}
\newcommand	 {\It}		{\mathbf{It}}
\newcommand	 {\IVPdo}	{\mathbf{IVPdo}}
\newcommand	 {\IVPbe}	{\mathbf{IVPbe}}

\newcommand  {\OL}		{\overline{\L}}
\newcommand  {\OM}		{\overline{\M}}
\newcommand  {\ON}		{\overline{\N}}
\newcommand  {\OS}		{\overline{\S}}
\newcommand  {\app}		{\downarrow}
\newcommand  {\up}		{\Uparrow}
\newcommand  {\dn}		{\Downarrow}
\newcommand  {\Ocup}	{\mathop{\overline{\cup}}}
\newcommand  {\wcat}	{+}
\newcommand  {\cat}		{\ast}
\newcommand  {\pcat}		{\mathop{\ast\ast}}
\newcommand  {\apcat}		{\mathop{\overrightarrow{\ast\ast}}}
\newcommand  {\x}		{\times}
\newcommand  {\Ocat}	{\mathop{\overline{\cat}}}
\newcommand  {\oa}		{\overline{\a}}
\newcommand  {\ob}		{\overline{\b}}
\newcommand  {\oab}		{\overline{\a\x\b}}
\newcommand  {\ot}		{\overline{t}}
\newcommand  {\oet}		{\overline{et}}
\newcommand  {\oett}		{\overline{(et)t}}
\newcommand  {\oeetet}		{\overline{(eet)et}}
\newcommand  {\super}	{\circ}
\newcommand  {\acatt}	{\overrightarrow{\cat}}
\newcommand  {\ccatt}	{\overleftarrow{\cat}}
\newcommand  {\acat}	{\mathop{\acatt}}
\newcommand  {\ccat}	{\mathop{\ccatt}}
\renewcommand{\a}		{\alpha}
\renewcommand{\b}		{\beta}
\newcommand	 {\g}		{\gamma}

\renewcommand{\l}		{\lambda}
\newcommand  {\e}		{\epsilon}
\newcommand  {\s}		{\sigma}
\renewcommand{\t}		{\tau}
\newcommand{\io}		{\iota}
\newcommand  {\p}		{\pi}
\renewcommand {\r}		{\rho}
\newcommand	 {\A}		{\mathcal{A}}

\newcommand	 {\C}		{\mathcal{C}}
\newcommand	 {\D}		{\mathcal{D}}
\newcommand	 {\E}		{\mathcal{E}}
\newcommand	 {\G}		{\mathcal{G}}

\newcommand	 {\I}		{\mathcal{I}}

\newcommand	 {\K}		{\mathcal{K}}
\renewcommand{\L}		{\mathcal{L}}
\newcommand	 {\M}		{\mathcal{M}}
\newcommand	 {\N}		{\mathcal{N}}
\newcommand	 {\R}		{\mathcal{R}}
\renewcommand{\P}		{\mathcal{P}}

\renewcommand{\S}		{\mathcal{S}}
\newcommand	 {\T}		{\mathcal{T}}

\newcommand	 {\RG}		{\R\G}

\newcommand  {\NOT}		{\sim}
\newcommand  {\OR}		{\vee}
\newcommand  {\AND}		{\wedge}
\newcommand  {\X}		{\mathop{\&}}
\newcommand  {\IMP}		{\rightarrow}
\newcommand  {\RA}		{\triangleright}
\newcommand  {\FA}		{\Rightarrow}
\newcommand	 {\VL}		{\boldsymbol{\L}}
\newcommand	 {\VK}		{\boldsymbol{\K}}

\newcommand	 {\CE}		{\boldsymbol{\hat \E}}

\newcommand	 {\eqd}		{=_\mathsf{def}}

\numberwithin{equation}{section}

\begin{document}     
\begin{center}
\begin{Large}
\textbf{Universal Higher Order Grammar\\}
\end{Large}
\vspace{24pt}
Victor Gluzberg\\
HOLTRAN Technology Ltd\\
gluzberg@netvision.net.il
\end{center}

\begin{center}
\textbf{Abstract\\}
\end{center}

We examine the class of languages that can be defined entirely in terms of provability in an extension of the sorted type theory ($Ty_n$) by embedding the logic of phonologies, without introduction of special types for syntactic entities. This class is proven to precisely coincide with the class of logically closed languages that may be thought of as functions from expressions to sets of logically equivalent $Ty_n$ terms.  For a specific sub-class of logically closed languages that are described by finite sets of rules or rule schemata, we find effective procedures for building a compact $Ty_n$ representation, involving a finite number of axioms or axiom schemata.  The proposed formalism is characterized by some useful features unavailable in a two-component architecture of a language model.  A further specialization and extension of the formalism with a context type enable effective account of intensional and dynamic semantics.

\section{Introduction}

Traditionally higher-order logic representation of natural language semantics \cite{montague74} had to be combined with an additional formalism to describe a syntactic structure of a language and a mapping between the two. This two-component architecture of a language model has not essentially changed with the invention and further rapid development of type-logical grammar \cite{lamb:math58,blackburn:1997}, i.e. a parallel logical formalism to describe a syntactic structure: in spite of the very close and deep correspondence between the two logics, their internal languages and semantics remain different. 

Several theoretical, methodological and technological issues are rooted in the two-component architecture of a language model, out of which it is important to mention here the following.

\begin{enumerate}

\item A language in such a model cannot express anything about another language of the same model nor, of course, about itself. This kind of expressiveness is, however, one of the characteristic capabilities of a natural language. 

\item Lack of lexical robustness.  As the semantic interpretation of a sentence in the two-component language model can be composed only from semantic interpretations of all its constituents, such a model fails to interpret an entire sentence if it contains even a single unrecognized (new) word. Some ad hoc add-ins to the formalism are the only known solution of the issue.

\item Lack of structural (syntactic) robustness. Unlike the semantic logic, which may be universally applicable for a wide variety of languages, the syntactic logic often requires special extensions in order to cover different languages and even specific structures in the same language. Capturing the semantic categories and their relationships in the very formalism (rather than in a concrete language model) makes it principally incapable of modeling language self-learning, that is derivation of new grammar rules from text samples.

\end{enumerate}

These, along with some other issues and needs motivated researches for generalizations based on a single logical system, applicable to both semantics and syntactic structure of a language together, such like \cite{kasper:rounds:1986}, \cite{king_89} and \cite{richter:diss}. Higher Order Grammar (HOG) \cite{pollard:hana:2003,pollard:2004:cg,pollard:2006} is probably the most recent implementation of the idea and the first one based entirely on the mainstream classical higher-order logic (HOL), traditionally applied only to the semantics of natural languages. The specific HOL employed in HOG comes in result of the following main steps:

\begin{enumerate}

\item introduction of a phonological base type and a constant denoting operation of concatenation of phonologies;
\item introduction of abstract syntax entity base types and constructors of derived types, comprising (together with these base types) a type \emph{kind} $\mathsf{SIGN}$;
\item introduction of another type kind $\mathsf{HIPER}$ for semantic interpretations, which consists of basic types of individuals and propositions and derived types - functional, product and of function to the boolean type;
\item embedding all the three logics - of phonologies, signs and hyperintensions - into a single HOL with fixing a correspondence between $\mathsf{SIGN}$ and $\mathsf{HIPER}$ types and adding the two special families of constants denoting functions from signs to phonologies and from signs to semantic interpretations of corresponding types.

\end{enumerate}

A concrete language is modeled in HOG by adding a set of non-logical axioms to postulate semantics and phonology of specific words and rules of their composition.

Due to introduction of the phonological type and constants for mapping signs to phonologies, HOG model is certainly better prepared for a formal account of lexical robustness.
Also, since directionality in HOG can be handled by the phonological interpretation, it is sufficient for it to have a single universal sign type constructor (implication) that partially addresses the issue of structural robustness.
However, as the syntactic and semantic type kinds are fully separated, HOG still cannot address the issue of self-expressiveness.

Another important implication of the HOG architecture, as it was noticed in \cite{gluzberg:2009}, is that semantic interpretation of any sentence or a constituent of it turns out to be represented not by a single HOL term, nor even by a set of some arbitrarily selected terms (in case of ambiguity), but by a whole class of logically equivalent terms, in a precise sense defined in the referenced work. Referring to languages revealing this semantic property as to \emph{logically closed} languages, one can say that all HOG-defined languages are necessarily logically closed. It was also explained in \cite{gluzberg:2009} why the inverse question, i.e. whether any logically closed language can be defined by a HOG with a given set of base $\mathsf{SIGN}$ types and type constructors, cannot be answered positively. This result gives another evidence of the limited robustness of the HOG model.

In the present work we examine what can be achieved by embedding into a HOL only the logic of phonologies, without introduction of special types for syntactic entities, but with use of non-logical constants of regular functional types to define syntax-to-phonology and syntax-to-semantics interfaces in a concrete language model. 
As such a language representation is, as well as HOG, entirely based on provability in the single HOL, all the languages it can model are still logically closed. 
The main result of this work consists of a proof of the inverse statement: any logically closed language can be represented in the proposed HOL framework. 
This justifies our referring to it as to Universal Higher-Order Grammar. 
Being fully free of embedded syntactic restrictions and of limitations on types of semantic interpretations, this formalism can efficiently address all the issues with the two-component architecture of a language model mentioned above and therefore may have several theoretical and practical implications.

The structure of the paper is as follows. 

After the introduction in Section 2 of basic notations and some assumptions on the axiomatization of a sorted type theory ($Ty_n$), in Section 3 we define the class of \emph{logically closed} languages and introduce a few basic operations that act within this class. 

In Section 4 we define an extension $Ty_n^\A$ of $Ty_n$ by interpreting one of its base types as \emph{symbolic}, which is quite similar to the phonological type of HOG, and introduce the notion of $Ty_n$-representability of a language, illustrated with a few preliminary examples. 

Then in Section 5 we prove that classes of $Ty_n$-representable and logically closed languages precisely coincide.

In the following sections we give explicit constructions of special $Ty_n$ representations for some important sub-classes of logically closed languages that are described by finite sets of rules or rule schemata. 

In Section 6 we consider so-called \emph{lexicons}, i.e. finitely generated logically closed languages, define special canonic $Ty_n$ representation and prove that a language has a canonic representation if and only if it is a lexicon. 

In Section 7 a $Ty_n$ representation is built for a \emph{recursive logically closed grammar} ($\RG$), defined as a tuple of logically closed languages linked with each other by a set of relationships expressed by the basic operations introduced in Section 3. As it becomes clear from illustrating examples, the language components of an $\RG$ stand for syntactic categories.

In Section 8 we consider a further specialization of $Ty_n^\A$ by interpreting another base type as a \emph{context}, similar (though not identical) to "state of affairs" or "World" types of \cite{gallin:1975} and \cite{pollard:2005:lacl}, respectively. This specialization allows to represent also intensional languages and further define \emph{instructive} and \emph{context-dependent} languages. Introduction of a few new language construction operations then allows us to generalize the previous results to context-dependent languages. We demonstrate some important capabilities of this formalism by examples of how it addresses pronoun anaphora resolution.

In Section 9 we introduce \emph{translation} and \emph{expression} operators that lead to a special language representation, revealing a useful property of \emph{partial translation}. 

In Section 10 we conclude by briefly summarizing and discussing the most important implications of the obtained results and outline some open issues.

\section{Notations}

Following \cite{gallin:1975}, we denote by $Ty_n$ a sorted type theory with a set of primitive types consisting of the truth type $t$ and individual types $e_1$, $e_2$, ... $e_n$ of $n > 1$ different sorts. For the first individual type we will also use a shorter alias $e \eqd e_1$. For sake of better visibility and compactness we combine the full syntax of $Ty_n$ from syntaxes of $Ty_2$ of \cite{gallin:1975} and $Q_0$ of \cite{andrews:1986} and extend it as follows.

\medskip \noindent
\textbf{Derived types}

\begin{enumerate}
\renewcommand{\labelenumi}{(\roman{enumi})}

\item If $\a$ and $\b$ are types, then $(\a\b)$ or just $\a\b$ is a \emph{functional type}, interpreted as type of functions from $\a$ to $\b$. The parenthesis are mandatory only in complex types, in order to express association in an order other than from right to left, for example: $et$, $tee$, $(et)e$ are equivalent to $(et)$, $(t(ee))$, $((et)e)$. Note that, unlike \cite{andrews:1986}, we write ``from'' and ``to'' types from left to right.

\item If $\a$ and $\b$ are types, then $(\a\x\b)$ or just $\a\x\b$ is a \emph{product type}, interpreted as type of pairs of elements from $\a$ and $\b$, so that, for example, $(\a\x\b)\g$ is equivalent to $\a\b\g$ and $\g(\a\x\b)$ to $\g\a\x\g\b$. Repeated product constructors are also associated from right to left, i.e. $(\a\x\b\x\g)$ is equivalent to $(\a\x(\b\x\g))$.

\end{enumerate}

\medskip \noindent
\textbf{Variables}

\medskip \noindent
$x$, $y$, $z$ subscripted by types and optionally superscripted by integer indices stand for \emph{variables} of corresponding types, for example: $y_{et}$, $x^1_{e}$, $x^2_{e}$.

\medskip \noindent
\textbf{Constants}

\medskip \noindent
\emph{Non-logical constants} are written as capital $\CC$ in bold,  subscripted by a type and superscripted by an index, like $\CC^0_t$. We do \emph{not} assume $Ty_n$ to necessarily have a constant $\CC^i_\a$ of a given type $\a$ for any index $i$. Rather, we assume the constants to be indexed in such a way that admits expanding $Ty_n$ by arbitrarily many new constants of any type. 
We also use arbitrary letters or words in bold (like $\Assert$ or $\Empty$) as mnemonic aliases for some constant that are assumed to have special semantics or/and satisfy some non-logical axioms in an extension of $Ty_n$.
Similar notations may also be introduced for some logical constants, i.e. pure bound terms, like, for example, identity $\II_{\a\a} \eqd \l x_\a x_\a$.

\medskip \noindent
\textbf{Terms}

\begin{enumerate}
\renewcommand{\labelenumi}{(\roman{enumi})}

\item Variables and constants comprise \emph{elementary terms}.

\item If $A_{\a\b}$ and $B_\a$ are terms of the corresponding types, then \emph{application} $A_{\a\b} \: B_\a$ is a term of type $\b$ denoting (in an interpretation of $Ty_n$) the value of the function denoted by $A_{\a\b}$ at the argument denoted by $B_\a$.

\item If $A_\b$ is a term of type $\b$ then \emph{lambda abstraction} $\l x_\a A_\b$ is a term of type $\a\b$ denoting a function whose value for any argument is the denotation of $A_\b$.

\item If $A_\a$ and $B_\a$ are terms of type $\a$, then $A_\a = B_\a$ is a term of type $t$, denoting the identity relation between elements of type $\a$.

\item If $A_\a$ and $B_\b$ are terms of the corresponding types, then the \emph{pair} $(A_\a, B_\b)$ is a term of type $\a\x\b$ denoting the ordered pair of denotations of $A_\a$ and $B_\b$; repeated operators $(,)$ are associated from right to left, so that a \emph{tuple} $(A^1_{\a_1}, A^2_{\a_2}, ... A^m_{\a_m})$ denotes the same as $(A^1_{\a_1}, (A^2_{\a_2}, (... A^m_{\a_m})...))$.

\item Finally, for a term $A_{\a\x\b}$ of a type $\a\x\b$, \emph{projections} $\p_1 A_{\a\x\b}$ and $\p_2 A_{\a\x\b}$ are terms of types $\a$ and $\b$ respectively, denoting the elements of the pair denoted by $A_{\a\x\b}$.

\end{enumerate}

We assume an axiomatization of $Ty_n$ be a straightforward generalization to case of $n > 2$ individual types of either the theory denoted as $Ty_n + D$ in \cite{gallin:1975} or the theory $Q_0$ of \cite{andrews:1986}. With either choice, for any type $\a$ the \emph{description operator} $\io_{(\a t)\a}$ and hence the "if-then-else" constant $\CC_{\a\a t \a}$ are available with the fundamental properties 
\[
\vdash \io_{(\a t)\a} \; \l x_\a (y_\a = x_\a) = y_\a
\]
\[
\vdash \CC_{\a\a t \a} \; x_\a \; y_\a \; \TT = x_a, \qquad
\vdash \CC_{\a\a t \a} \; x_\a \; y_\a \; \FF = y_a.
\]

All the usual Boolean connectives, including terms $\FF$ and $\TT$ denoting the false and true values, and quantifiers can be defined in $Ty_n$ exactly the same way as in $Ty_2$ and $Q_0$. In addition, we will also employ the following notational shortcuts:
\begin{eqnarray*}
        A_\a \neq B_\a \;&\eqd&\; \NOT (A_\a = B_\a), \\
         \NOT A_{\a t} \;&\eqd&\; \l x_\a \NOT ( A_{\a t} \; x_\a),\\
A_{\a t} \OR B_{\a t}  \;&\eqd&\; \l x_\a (A_{\a t} \: x_\a \OR B_{\a t} \: x_\a), \\
A_{\a t} \AND B_{\a t} \;&\eqd&\; \l x_\a (A_{\a t} \: x_\a \:\AND\: B_{\a t} \: x_\a), \\
A_{\a t} \X B_{\b t} \;&\eqd&\; \l x_\a \l x_\b (A_{\a t} \: x_\a \:\X\: B_{\b t} \: x_\b), \\
A_{\a t} \IMP B_{\b t} \;&\eqd&\; \forall  x_\a (A_{\a t} \: x_\a \IMP B_{\a t} \: x_\a), \\
A_\a | B_\a \;&\eqd&\; \CC_{\a\a t\a} \: A_\a \: B_\a,
\end{eqnarray*}
and, for an arbitrary binary operator $O$:
\begin{eqnarray*}
(O\:B_\b) \;&\eqd&\; \l x_\a(x_\a\:O\:B_\b), \\
(A_\a\:O) \;&\eqd&\; \l x_\b(A_\a\:O\:x_\b), \\
      (O) \;&\eqd&\; \l x_\a\l x_\b(x_\a\:O\:x_\b),
\end{eqnarray*}
(an operator sign in such shortcuts might be subscripted by a type, like, for example, $(=_{e e t})$ to denote specifically $\l x_e \l y_e (x_e = y_e)$.

In the meta-language:
\begin{enumerate}
\renewcommand{\labelenumi}{(\roman{enumi})}
\item the $\vdash$ sign denotes provability in $Ty_n$ or an extension of it, specified by a context;
\item the notation $A_\a(B_\b)$ stands for the result of substituting all occurrences of a free variable $x_\b$ in a term $A_\a$ by a term $B_\b$ free for $x_\b$ in $A_\a$;
\item if $M$ is a model of $Ty_n$ and $a$ - an assignment of variables in this model, then $\llb\cdot\rrb_M$ denotes interpretation of a constant in $M$ and $\llb\cdot\rrb_{M,a}$ denotes the value of an arbitrary term assigned to it by $a$ in $M$;
\item $\Longrightarrow$ and $\Longleftrightarrow$ express implication and logical equivalence.
\end{enumerate}

\section{Logically closed languages}

The two fundamental formalizations of the notion of language with which we deal in this work - $\a$-\emph{language} and \emph{logically closed} $\a$-\emph{language} have been introduced and informally discussed in \cite{gluzberg:2009}. We reproduce these formal definitions here with the current notations for more convenient references.

\begin{defn} \label{defn3.1}
Let $\A$ be a finite alphabet $\{ a_1,\:a_2,\:...\:a_N \}$, let $\A^*$ denote the set of all words over this alphabet and let $\T_\a$ denote the set of all $\a$-terms of $Ty_n$. An $\a$-\emph{language} is a relation $\L \subset \A^* \otimes \T_\a$.
\end{defn}

Referring to words over alphabet $\A$ as ``expressions'' and $\a$-terms as ``$\a$-meanings,'' one can say that an $\a$-language is a set of pairs of expressions and their $\a$-meanings. 
A set of words $\L \subset \A^*$ can then be considered as a language for the unit type meaning.
In general case, the projection of $\L$ to $\A^*$ is the set of all valid expressions of $\L$, to which set we refer as to \emph{domain} of the language and the projection of $\L$ to $\T_\a$ is the set of all meanings $\L$ can express, to which set we refer as to \emph{range} of the language.

A trivial particular case of an $\a$-language is a singleton $\{ (w, A_\a) \}$, where $w \in \A^*$.

If $\L$ and $\K$ are $\a$-languages, then their union $\L \cup \K$ is a new $\a$-language that contains all expressions and corresponding meanings of both $\L$ and $\K$.

The following definitions introduce some further operations to build complex languages from simpler ones.

\begin{defn} \label{defn3.2}
If $\L$ is an $\a$-language and $\K$ is a $\b$-language, then their \emph{language concatenation} is the $\a\x\b$-language 
\[
\L \cat \K \eqd \{(uv, (A_\a, B_\b)) \:|\: (u, A_\a) \in \L \;\AND\; (v, B_\b) \in \K \}
\]
where $uv$ denotes concatenation of words $u$ and $v$.
\end{defn}

\begin{defn} \label{defn3.3}
For a given $\a$-language $\L$ and a relation $\R \subset \T_\a \otimes \T_\b$, the \emph{semantic rule application} is the $\b$-language
\[
\L \RA \R \:\eqd\: \{(w, B_\b) \:|\: (w, A_\a) \in \L \;\AND\; (A_\a, B_\b) \in \R \}.
\]
\end{defn}

Thus, application of a semantic rule cannot extend the domain of a language, but only maps the set of meanings of every its expression to another set (generally of a different type); if the new set turns out to be empty, then the corresponding expression is "filtered out" from the domain of resulting language. An important example of a semantic rule application is "folding" product-type meanings of a concatenation of two languages to meanings of a scalar type. Note that a folding rule which also filters out some expressions from the language concatenation in fact can check an agreement between the concatenated constituents (at the semantic level).

The above definition of an $\a$-language seems to provide the most general formalization for the notion of a language whose distinct meanings are understood as syntactically different (although possibly logically equivalent) $Ty_n$ terms of a certain type. For example, the representation of $Ty_n$ formulas in the $Ty_n$ language is itself a $t$-language. In the application to natural languages, however, a more restricted formalization might be more suitable:

\begin{defn} \label{defn3.4}
A \emph{logically closed} $\a$-language is an $\a$-language $\L$ such that whenever $(w, B_\a) \in \L, \; (w, C_\a) \in \L$ and $\vdash A_\a = B_\a \OR A_\a = C_\a$ then $(w, A_\a) \in \L$ also. 
A minimal logically closed $\a$-language $\OL$ which includes a given arbitrary $\a$-language $\L$ is said to be its \emph{logical closure}.
\end{defn}

This definition actually captures the two important features of a logically closed $\a$-language:

\begin{enumerate}
\item If an expression $w$ in the language has a meaning $B_\a$, then it also has every meaning $A_\a$ logically equivalent to $B_\a$
\item If an expression $w$ is ambiguous, i.e. has at least two distinct meanings $B_\a$ and $C_\a$ being \emph{not} logically equivalent, then it also has every meaning $A_\a$ which is provable to be equal either $B_\a$ or $C_\a$.
 \end{enumerate}

Therefore, every valid expression of a logically closed language is associated not with a single, nor even with a set of arbitrarily selected terms (in case of ambiguity), but with a whole class of logically equivalent terms. A precise formulation of this interpretation follows.

\begin{defn} \label{defn3.5}
A set $\M \subset \T_\a$ is said to be \emph{logically closed} iff whenever $B_\a \in \M,\; C_\a \in \M$ and
\begin{equation} \label{eq3.1}
\vdash A_\a = B_\a \OR A_\a = C_\a
\end{equation}
then $A_\a \in \M$ also.
A minimal logically closed set $\OM \subset \T_\a$ which includes an arbitrary set $\M \subset \T_\a$ is said to be its \emph{logical closure}.
If in addition, $\N \subset \T_\a$ and $\OM = \ON$, we say the two sets $\M$ and $\N$ be \emph{logically equivalent} and denote this relation as $\M \simeq \N$.
\end{defn}

It is readily seen that $\simeq$ is an equivalence relation in the power set \mbox{$\P(\T_\a)$} and therefore a logically closed $\a$-language might be defined equivalently as a function $\L : \A^* \rightarrow \P(\T_\a)/\simeq$.

The simplest non-empty logically closed $\a$-language is a \emph{logical singleton}
\begin{equation} \label{eq3.2}
\OS = \{ w \} \otimes \overline{\{ A_\a \}}.
\end{equation}

As it has been shown in \cite{gluzberg:2009}, any language defined by a Higher Order Grammar (HOG) \cite{pollard:hana:2003,pollard:2004:cg,pollard:2006} is necessarily logically closed.

Note that the class of logically closed languages is not closed under the union operation nor under operations of language concatenation and semantic rule applications defined above. Similar operations that act within this class can be defined as follows.

\begin{defn} \label{defn3.6}
If $\L$ and $\K$ are $\a$-languages, then their \emph{logical join} is an $\a$-language $\L \Ocup \K \eqd \overline{\L \cup \K}$.
\end{defn}

\begin{defn} \label{defn3.7}
If $\L$ is an $\a$-language and $\K$ is a $\b$-language, then their \emph{logical concatenation} is an $\a\x\b$-language \mbox{$\L \Ocat \K \eqd \overline{\L \cat \K}$}.
\end{defn}

\begin{defn} \label{defn3.8}
A semantic rule $\R \subset \T_\a \otimes \T_\b$ is said to be \emph{logically closed} if the full image of any logically closed set $\M \subset \T_\a$ under the relation $\R$ is logically closed.
\end{defn}

Thus, a logical join and logical concatenation of arbitrary languages, as well as application of a logically closed semantic rule to any language of the matching type, are all logically closed languages.

The following associativity and monotonicity properties follow directly from the above definitions.

\begin{lem} \label{lem3.1}
For any $\a$-languages $\K, \L$, a $\b$-language $\M$ and a rule \\ $\R \subset \T_\a \otimes \T_\b$,
\begin{eqnarray*}
(\K \Ocup \L) \Ocat \M \;&=&\; \K \Ocat \M \;\Ocup\; \L \Ocat \M, \\
\M \Ocat (\K \Ocup \L) \;&=&\; \M \Ocat \K \;\Ocup\; \M \Ocat \L, \\
(\K \Ocup \L) \RA \R \;&=&\; \K \RA \R \;\Ocup\; \L \RA \R, \\
\K \subset \L \;&\Longrightarrow&\; \K \Ocat \M \subset \L \Ocat \M, \\
\K \subset \L \;&\Longrightarrow&\; \M \Ocat \K \subset \M \Ocat \L, \\
\K \subset \L \;&\Longrightarrow&\; \K \RA \R \subseteq \L \RA \R.
\end{eqnarray*}
\end{lem}

\begin{defn} \label{defn3.9}
A logically closed set $\OM \subset \T_\a$ is said to be \emph{finitely generated} if it is the logical closure of a finite set $\M \subset \T_\a$.
\end{defn}

\begin{defn} \label{defn3.10}
A logically closed semantic rule $\R \subset \T_\a \otimes \T_\b$ is said to be \emph{finitely ambiguous} if for any term $A_\a$, the full image of $\overline{\{A_\a\}}$ under the relation $\R$ is finitely generated.
\end{defn}

An important class of finitely ambiguous semantic rules is given by the following

\begin{defn} \label{defn3.11}
A logically closed semantic rule $\R \subset \T_\a \otimes \T_\b$ is said to have a \emph{canonic representation in} $\D_\a \subset \T_\a$ if there exists a $Ty_n^\A$ term $R_{\a\b t}$ such that
\[
(A_\a, B_\b) \in \R \qquad \Longleftrightarrow \qquad \vdash R_{\a\b t} \: A_\a \: B_\b
\]
and for any $A_\a \in \D_\a$ 
\[
\vdash R_{\a\b t} \: A_\a = \bigvee_{i = 1}^{k} (= B^i_\b),
\]
where the terms $B^1_\b, \:...\: B^k_\b$ as well as the number $k$ may depend on $A_\a$ and in the case $k = 0$ the disjunction is reduced to $\l x_\b \FF$. We further qualify a rule with a canonic representation in $\D_\a \subset \T_\a$ as \emph{non-degenerate, degenerate, unambiguous} or \emph{ambiguous in} $\D_\a$, if $k > 0$ for all $A_\a \in \D_\a$, $k = 0$ for some $A_\a \in \D_\a$, $k = 1$ for all $A_\a \in \D_\a$ or $k > 1$ for some $A_\a \in \D_\a$, respectively.
\end{defn}

An important example of a semantic rule with a canonic representation in any domain is given by a particular case where
\[
R_{\a\b t} \eqd \l x_\a (= F_{\a\b} \: x_\a)
\]
which is referred to below as a \emph{functional} semantic rule. Note that a functional rule is both non-degenerate and unambiguous in any domain.

\section{Symbolic type and $Ty_n$-representable \\ languages}

Consider an extension $Ty_n^\A$ of $Ty_n$ with the following set of non-logical axioms, where $s$ denotes the primitive type $e_n$ and $A_s \wcat B_s \eqd \CC^0_{sss} \: A_s \: B_s$:
\begin{equation} \label{eq4.1}
\forall x_s ((x_s \wcat \CC^0_s = x_s) \;\AND\; (\CC^0_s \wcat x_s = x_s)),
\end{equation}
\begin{equation} \label{eq4.2}
\forall x_s \forall y_s \forall z_s ((x_s + y_s) + z_s = x_s + (y_s + z_s)),
\end{equation}
\begin{equation} \label{eq4.3}
\forall x_s \forall y_s \forall z_s (x_s + \CC^i_s + y_s \neq x_s + \CC^j_s + z_s), \quad\text{for all}\; 1 \leq i < j \leq N.
\end{equation}
It is easy to see that these axioms are valid in a model $M$ with the $s$-domain $\D_s \eqd \A^*$ and interpretation $\llb\cdot\rrb_M$ such that $\llb\CC^0_s\rrb_M$ is the empty word $\e \in \A^*$, $\llb\CC^i_s\rrb_M$ is $a_i \in \A$ for $1 \leq i \leq N$, all the rest of the $s$-constants are interpreted by compound words from $\A^*$ and, finally, the operation $(\wcat)$ is interpreted as the operation of word concatenation. This observation proves the consistency of $Ty_n^\A$ and justifies our referring to the type $s$ as \emph{symbolic type}.
Note that formally, with the accuracy up to the additional axiom \ref{eq4.3}, the symbolic type is quite similar to the phonological type of HOG \cite{pollard:2006}. We use the different term here in order to reflect the additional axiomatic restriction, expressing distinction of the alphabet symbols, significant in any context, and also keeping in mind that generally they may stand for symbols of an arbitrary nature, for example, graphic, as well as phonetic, in which case $Ty_n^\A$ might be further extended to accommodate more symbol aggregation operations in addition to the linear concatenation.

$Ty_n^\A$ allows some $\a$-languages to be naturally defined by its terms or possibly by the terms of its extension $Ty_n^{\A+}$ with a set of additional non-logical constants. Indeed, define the mapping $T_A: \A^* \rightarrow \T_s$ as follows
\begin{equation} \label{eq4.4}
T_A(\e) = \CC^0_s, \qquad T_A(a_i) = \CC^i_s, \qquad T_A(a_i w) = \CC^i_s \wcat T_A(w),
\end{equation}
and let $\Delta$ be a consistent set of $Ty_n^{\A+}$ formulas. Then the condition
\[
\Delta \vdash L_{s \a t} \: T_A(w) \: A_\a
\]
for any $Ty_n^{\A+}$ term $L_{s \a t}$ defines an $\a$-language ($\vdash$ here and everywhere below denotes provability in $Ty_n^{\A+}$).

\begin{defn} \label{defn4.1}
Let $Ty_n^{\A+}$ be an extension of $Ty_n^\A$, $\Delta$ a consistent set of $Ty_n^{\A+}$ formulas and $L_{s \a t}$ a $Ty_n^{\A+}$ term. An $\a$-language $\L$ is said to be \emph{represented by  $(\Delta, L_{s \a t})$} if for any $Ty_n^\A$ term $A_\a$ and $w \in \A^*$
\[
(w, A_\a) \in \L \qquad \Longleftrightarrow \qquad \Delta \vdash L_{s \a t} \: T_A(w) \: A_\a.
\]
\end{defn}

Note that in the case of a finite set $\Delta$, by the Deduction Theorem, a language represented by $(\Delta, L_{s \a t})$ is also represented by $(\empty,\; \l x_s \l x_\a (D_t \IMP L_{s \a t} \: x_s \: x_\a))$ where $D_t$ is the conjunction of all formulas of $\Delta$ and the variables $x_s, x_\a$ are not free in $D_t$. In this case we say the language to have a \emph{compact $Ty_n$ representation} or, alternatively, to be \emph{represented by a term}.

\begin{defn} \label{defn4.2}
An $\a$-language $\L$ is said to be \emph{representable in $Ty_n^\A$}, or \emph{$Ty_n$-representable} for short, if there exists an extension $Ty_n^{\A+}$ of $Ty_n^\A$, a consistent set $\Delta$ of $Ty_n^{\A+}$ formulas and a $Ty_n^{\A+}$ term $L_{s \a t}$ such that $\L$ is represented by $(\Delta, L_{s \a t})$.
\end{defn}

Here are a few simple examples of $Ty_n$-representable languages:
\begin{itemize}
\item the term $(= T_A(w)) \X (= A_\a)$ represents a logical singleton (\ref{eq3.2});
\item the term $\l x_s (= \CC^0_{s \a} \: x_s)$ represents an $\a$-language where every word $w$ is a valid expression for the application $\CC^0_{s \a} \: T_A(w)$;
\item the term $(=_{sst})$ represents an $s$-language that realizes the mapping (\ref{eq4.4}), i.e. associates with every word $w$ a single ``meaning'' $T_A(w)$.
\end{itemize}

In case of alphabet $\A$ consisting of symbols that can be typed in this paper (including a space), a more convenient notation for mapping $T_A$ can be introduced: if $w$ is an arbitrary string of such symbols, let $/w/ \eqd T_A(w)$. Then, by definition \ref{defn4.1}, for arbitrary strings $u, v$, $/u/ + /v/ = /uv/$; for example: $/\text{c}/+/\text{a}/+/\text{r}/ = /\text{car}/$. We will make use of this practical notation in some of subsequent sections.

\section{$Ty_n$ representation existence theorem}

We now show that an $\a$-language is $Ty_n$-representable if and only if it is logically closed.

\begin{prop} \label{prop5.1}
If an $\a$-language $\L$ is $Ty_n$-representable, then it is logically closed.
\end{prop}
\begin{proof}
Let $(\Delta, L_{s\a t})$ be a representation of $\L$, $w \in \A^*$,
\[
\Delta \vdash L_{s\a t} \: T_A(w) \: B_\a, \quad \Delta \vdash L_{s\a t} \: T_A(w) \: C_\a \quad \text{and} \quad \vdash A_\a = B_\a \OR A_\a = C_\a.
\]
From this derivation, by the metatheorems
\begin{equation} \label{eq5.1}
\vdash A_t \OR B_t = (A_t | B_t) \: B_t,
\end{equation}
\begin{equation} \label{eq5.2}
\vdash (P_{\a\b} \: A_\a | P_{\a\b} \: B_\a) \: C_t = P_{\a\b} ((A_\a | B_\a) \: C_t),
\end{equation}
it follows that
\[
\vdash A_\a = (B_\a | C_\a) \: (A_\a = C_\a)
\]
and
\[
\vdash L_{s\a t} \: T_A(w) \: A_\a = (L_{s\a t} \: T_A(w) \: B_\a \:|\: L_{s\a t} \: T_A(w) \: C_\a) \: (A_t = C_\a).
\]
Thus the first two derivations imply \;$\Delta \vdash L_{s\a t} \: T_A(w) \: A_\a$, that is, $(w, A_\a) \in \L$ also.
\end{proof}

\begin{lem} \label{lem5.2}
In any $Ty_n^\A$ model, the $s$-domain $\D_s$ contains a subset $\D'_s$ isomorphic to $\A^*$ with respect to concatenation operations and such that $\llb T_A(w)\rrb_M \in \D'_s$ for any $w \in \A^*$.
\end{lem}
\begin{proof}
Let $M$ be a $Ty_n^\A$ model and $c_i \eqd \llb\CC^i_s\rrb_M$ for $i = 0,...\:N$. 
Define the mapping $V_A : \A^* \rightarrow \D_s$ as follows:
\[
V_A(\e) = c_0, \quad V_A(a_i) = c_i, \qquad V_A(a_i + w)   = c_i \;\llb+\rrb_M\; V_A(w).
\]
Since axioms \ref{eq4.1} - \ref{eq4.3} are valid in $M$, $V_A$ is a one-to-one mapping and
\[
V_A(u + v) = V_A(u) \;\llb+\rrb_M\; V_A(v).
\]
Thus the full image $\D'_s$ of $\A^*$ under this mapping is isomorphic to $\A^*$ and, by Definition \ref{eq4.4}, $\llb T_A(w)\rrb_M \in \D'_s$ for any $w \in \A^*$.
\end{proof}

This lemma actually allows to identify $\llb T_A(w)\rrb_M$ with $w$, $\llb+\rrb_M$ - with word concatenation and the subset $\D'_s$ - with $\A^*$ in any $Ty_n^\A$ model $M$.

\begin{lem} \label{lem5.3}
If $\OM \subset \T_\a$ is logically closed, $A^1_\a \in \OM, \:...\: A^k_\a \in \OM$ and $\vdash A_\a = A^1_\a \:...\: \OR A_\a = A^k_\a$, then $A_\a \in \OM$ also.
\end{lem}
\begin{proof}
The proof is by induction on $k$. 
For $k \leq 2$, the statement follows directly from Definition \ref{defn3.5}. 
For $k > 2$, by the metatheorems \ref{eq5.1}, \ref{eq5.2} and
\[
\vdash (((A_\a | B_\a) \: C_t) = A_\a) \OR (((A_\a | B_\a) \: C_t) = B_\a),
\]
we have
\[
\vdash A_\a = B_\a \OR A_\a = A^k_\a,
\]
where
\[
B_\a \eqd (...(A^1_\a | \:... | A^{k-2}_\a) (A_\a = A^{k-2}_\a) | A^{k-1}_\a) (A_\a = A^{k-1}_\a)
\]
and therefore
\[
\vdash B_\a = A^1_\a \:...\: \OR B_\a = A^{k-1}_\a.
\]
Thus, if this implies that $B_\a \in \OM$, then $A_\a \in \OM$ also, according to Definition \ref{defn3.5}.
\end{proof}

\begin{prop} \label{prop5.4}
Let an $\a$-language $\L$ be logically closed, let the constant $\CC^0_{s\a t}$ belong to $Ty_n^{\A+}$, but not to $Ty_n^\A$, and let $\Delta_\L$ be the minimal set of $Ty_n^{\A+}$ formulas such that, whenever $(w, A_\a) \in \L$, the formula $\CC^0_{s\a t} \: T_A(w) \: A_\a$ belongs to $\Delta_\L$.
Then $(\Delta_\L, \CC^0_{s\a t})$ is a representation of $\L$.
\end{prop}
\begin{proof}
Consistency of the set $\Delta_\L$ follows from the fact that all its formulas are valid in a model where the constant $\CC^0_{s\a t}$ is interpreted by a function taking the true value for all its arguments.

The implication \;$(w, A_\a) \in \L \;\Longrightarrow\; \Delta_\L \vdash \CC^0_{s \a t} \: T_A(w) \: A_\a$\; follows directly from the definition of $\Delta_\L$.

To prove the opposite implication, assume that $\Delta_\L \vdash \CC^0_{s \a t} \: T_A(w) \: A_\a$ and $\Delta_w$ is the (necessarily finite) subset of axioms from $\Delta_\L$ participating in a proof of this derivation, so that 
\begin{equation} \label{eq5.3}
\Delta_w \vdash \CC^0_{s \a t} \: T_A(w) \: A_\a
\end{equation}
as well. 
Let $A^1_\a, ...\: A^k_\a$ be all those terms that occur in the formulas \\ $\CC^0_{s \a t} \: T_A(w) \: A^i_\a$ belonging to $\Delta_w$.
If $(w, A_\a) \not\in \L$, then there exist a $Ty_n^\A$ model $M$ and an assignment of variables $a$ such that $\llb A_\a\rrb_{M,a} \neq \llb A^i_\a\rrb_{M,a}$ for all $i = 1,...\: k$.
Indeed, due to the Completeness Theorem, we would otherwise have
\[
\vdash A_\a = A^1_\a \:...\: \OR A_\a = A^k_\a
\]
and therefore, by Lemma \ref{lem5.3} and the assumption of logical closure of $\L$, $(w, A_\a) \in \L$.
Consider an extension $M^+$ of the model $M$ for $Ty_n^{\A+}$ where $\llb\CC^0_{s \a t}\rrb_M$ is the function taking the true value for all its arguments except for $(w, \llb A_\a\rrb_{M,a})$ (where $w$ stands for $\llb T_A(w)\rrb_M$, according to Lemma \ref{eq5.2}).
All formulas of $\Delta_w$ are satisfied by the assignment $a$ in $M^+$, but $\CC^0_{s \a t} \: T_A(w) \: A_\a$ is not.
By the Soundness Theorem, this contradicts \ref{eq5.3} and thus proves that $(w, A_\a) \in \L$.
\end{proof}

Propositions \ref{prop5.1} and \ref{prop5.4} immediately imply
\begin{thm} \label{thm5.5}
An $\a$-language is $Ty_n$-representable iff it is logically closed.
\end{thm}

Note that this general result is not constructive in the sense that it does not entail an effective procedure for actually building a $Ty_n$ representation for an infinite language that might be defined by a finite set of rules or rule schemata. 
At the same time, it neither implies the uniqueness of the $Ty_n$ representation and thus does not preclude a representation of such a language that may be compact or described by a finite set of axiom schemata.
We will find such representations for some important special classes of logically closed languages in the following sections.

\section{Lexicons and canonic representation}

\begin{defn} \label{defn6.1}
An $\a-lexicon$ is a logically closed $\a$-language which is the logical closure of a finite $\a$-language.
\end{defn}

\begin{defn} \label{defn6.2}
An $\a$-language is said to have a \emph{canonic representation} if it is represented by a term $L_{s \a t}$ such that
\[
\vdash L_{s \a t} = (=A^1_s) \X (=A^1_\a) \:...\: \OR (=A^l_s) \X (=A^l_\a)
\]
where all $A^i_s$ are $s$-constants and $A^i_\a$ are $Ty_n$ terms.
\end{defn}

In this section we show that any $\a$-lexicon has a canonic representation.

\begin{lem} \label{lem6.1}
$\vdash T_A(u) = T_A(v)$ \:iff\: $u = v$ \;and\; $\vdash T_A(u) \neq T_A(v)$ \:iff\: $u \neq v$.
\end{lem}
\begin{proof}
According to \ref{eq4.1} -- \ref{eq4.3}, if $m = l$ and $i_1 = j_1$, ... $i_l = j_l$, then
\[
\vdash \CC^{i_1}_s \:...\: \wcat \CC^{i_l}_s = \CC^{j_1}_s \:...\: \wcat \CC^{j_m}_s
\]
and if at least one of these conditions does not hold, then
\[
\vdash \CC^{i_1}_s \:...\: \wcat \CC^{i_l}_s \neq \CC^{j_1}_s \:...\: \wcat \CC^{j_m}_s.
\]
This proves both statements.
\end{proof}

\begin{lem} \label{lem6.2}
Let the mapping $T_{A \a}$ from the set of $\a$-languages to the power set $\P(\T_s \otimes \T_\a)$ be defined by
\[
T_{A \a}(\L) = \{ (T_A(w), A_\a) \: | \: (w, A_\a) \in \L\}.
\]
Then $\OL$ is the logical closure of $\L$ iff $T_{A \a}(\OL)$ is the logical closure of $T_{A \a}(\L)$.
\end{lem}
\begin{proof}
As usual for any closure constructed by a ternary relation, a term belongs to a logical closure iff there exists a finite proof tree where this term is the root, every non-leaf node $A_\a$ has a pair of children $B_\a$, $C_\a$ that satisfy the relation (\ref{eq3.1}) and all leaves belong to the original set. 
Similarly, a pair from $\A^* \otimes \T_\a$ belongs to the logical closure of an $\a$-language iff there exists a finite proof tree where this pair is the root, every non-leaf node $(w, A_\a)$ has a pair of children $(w, B_\a)$, $(w, C_\a)$ where $B_\a$ and $C_\a$ satisfy the relation (\ref{eq3.1}) and leaves belong to the original language. 
Replacing every node $(w, A_\a)$ of such a tree with $(T_A(w), A_\a)$ by Lemma \ref{lem6.1} converts it to a proof tree for the logical closure of $T_{A \a}(\L)$. Thus if $\OL$ is the logical closure of $\L$, then $T_{A \a}(\OL)$ is the logical closure of $T_{A \a}(\L)$. 
To prove the opposite implication, note that due to the tautology
\[
\vdash A_t \AND C_t \;\OR\; B_t \AND D_t \;\IMP\; C_t \OR D_t,
\]
the relation
\[
\vdash A_s = B_s \;\AND\; A_\a = B_\a \;\;\OR\;\; A_s = C_s \;\AND\; A_\a = C_\a
\]
implies (\ref{eq3.1}). Therefore replacing every node $(A_s, A_\a)$ of a proof tree for the logical closure of $T_{A \a}(\L)$ with $(w, A_\a)$, where $w$ is the expression occurring in the root of the original tree, converts it to a proof tree for $\OL$, by Lemma \ref{lem6.1}.
\end{proof}

\begin{defn} \label{defn6.3}
A set $\M \subset \T_\a$ is said to be represented by a term $M_{\a t}$ if \;$A_\a \in \M \;\Longleftrightarrow \; \vdash M_{\a t} A_\a$.
\end{defn}

\begin{lem} \label{lem6.3}
The logical closure $\OM$ of a finite set \;$\M = \{A^1_\a, \:...\: A^k_\a\}$\; is represented by the term \;$M_{\a t} \eqd (= A^1_\a) \:...\: \OR (=A^k_\a)$.
\end{lem}
\begin{proof}
$\vdash M_{\a t} B_\a$ and $\vdash M_{\a t} C_\a$ together with (\ref{eq3.1}) imply $\vdash M_{\a t} A_\a$ by the metatheorem 
\[
\vdash (A_\a = B_\a \AND M_{\a t} B_\a) 
\IMP M_{\a t} A_\a
\]
(also applied to $C_\a$) and the tautology 
\[
\vdash (A_t \OR B_t) \AND D_t \AND E_t \IMP (A_t \AND D_t \OR B_t \AND E_t).
\]
Thus, the set defined by $\vdash M_{\a t} A_\a$ is logically closed and therefore, according to Definition \ref{defn3.5}, includes $\OM$. 
The opposite inclusion follows directly from Lemma \ref{lem5.3}.
\end{proof}

\begin{thm} \label{thm6.4}
A logically closed language has a canonic representation iff it is a lexicon.
\end{thm}
\begin{proof}
Let \;$\L = \{ (w_1, A^1_\a), \:...\: (w_l, A^l_\a) \}$\; and
\[
L_{s \a t} \eqd (=A^1_s) \X (=A^1_\a) \:...\: \OR (=A^l_s) \X (=A^l_\a),
\]
where $A^i_s \eqd T_A(w_i)$. According to Lemma \ref{lem6.3}, the term $L_{s \a t}$ represents the logical closure of $T_{A \a}(\L)$ and therefore, according to Lemma \ref{lem6.2}, also the logical closure $\OL$.
\end{proof}

\section{Recursive languages and representations}

Given a finite set of lexicons, one can obtain an arbitrarily complex logically closed language by recursively applying the operations of logical join, logical concatenation and application of semantic rules to the initial lexicons.
In this section we will see how a compact $Ty_n$ representation of such a recursive language can be built.

\begin{defn} \label{defn7.1}
Let $\L_1, ... \L_m$ be logically closed languages and let one of the following be true for every $1 \leq i \leq m$:
\begin{enumerate}
\renewcommand{\labelenumi}{(\roman{enumi})}
\item $\L_i$ is a lexicon and $L^i_{s\a_i t}$ is its canonic representation
\item there exist such $1 \leq j(i) \leq m$ and $1 \leq k(i) \leq m$ that $\L_i = \L_{j(i)} \Ocup \L_{k(i)}$
\item there exist such $1 \leq  j(i) \leq m$ and $1 \leq k(i) \leq m$ that $\L_i = \L_{j(i)} \Ocat \L_{k(i)}$
\item there exists such $1 \leq j(i) \leq m$ that $\L_i = \L_{j(i)} \RA \R_i$, where a logically closed semantic rule $\R_i$ has a canonic representation $R^i_{\a_{j(i)}\a_i t}$ in the range of $\L_{j(i)}$.
\end{enumerate}
The full set $\RG$ of these conditions comprises a \emph{recursive grammar} for the languages $\L_1, ... \L_m$.
\end{defn}

For a given recursive grammar, let $\VL$ denote the language tuple $(\L_1, ... \L_m)$, let $\VL^0$ denote a tuple of lexicons such that $\L^0_i \eqd \L_i$ whenever $\L_i$ is a lexicon and $\L^0_i = \empty$ otherwise and let $\VL'$ denote a language tuple such that $\L'_i = \empty$ whenever $\L_i$ is a lexicon and $\L'_i$ is the right-hand part of the equality in condition (ii), (iii) or (iv) otherwise. Finally, let $\CE$ denote the operator which produces $\VL'$ for a given $\VL$, so that the recursive grammar can be written as
\begin{equation} \label{eq7.1}
\VL = \VL^0 \Ocup \CE \VL,
\end{equation}
where $\Ocup$ stands for a component-wise logical join of language tuples of the same arity and types.

\begin{prop} \label{prop7.1}
Equation \ref{eq7.1} is satisfied by
\begin{equation} \label{eq7.2}
\VL = \mathop{\overline{\bigcup}}_{i = 0}^{\infty} \VL^i,
\end{equation}
where
\begin{equation} \label{eq7.3}
\VL^i = \VL^0 \Ocup \CE \VL^{i-1}, \quad \text{for all} \; i \geq 1.
\end{equation}
\end{prop}
\begin{proof}
As the operator $\CE$ is a combination of logical joins, logical concatenations and semantic rule applications, it inherits the monotonicity property
\begin{equation} \label{eq7.4}
\VL \subseteq \VK \quad \Longrightarrow \quad \CE \VL \subseteq \CE \VK,
\end{equation}
where $\VK$ is a language tuple of the same arity and types as $\VL$ and $\subseteq$ stands for component-wise inclusion.
From this and \ref{eq7.3}, by induction on $i$ it follows that
\[
\VL^i \supset \VL^{i-1}, \quad \text{for all} \; i \geq 1.
\]
Therefore, for any tuple $\ee \eqd ((w_1, A_{\a_1}), \:...\: (w_m, A_{\a_m}))$ there exists an index $i(\ee)$ such that
\[
\ee \in \CE \: \mathop{\overline{\bigcup}}_{i=0}^{\infty} \VL^i \quad \Longrightarrow \quad \ee \in \CE \VL^{i(\ee)},
\]
where $\in$ stands for component-wise membership.
From here, for the tuple $\VL$ given by \ref{eq7.2} we have that
\[
\ee \in \VL^0 \Ocup \CE \VL \quad \Longrightarrow \quad \ee \in \VL^0 \Ocup \CE \VL^{i(\ee)} = \VL^{i(\ee)+1},
\]
that is, $\VL^0 \Ocup \CE \VL \subseteq \VL$. 
On the other hand, from the same property \ref{eq7.4} we have that $\VL^0 \Ocup \CE \VL \supseteq \VL^0 \Ocup \CE \VL^i$ for any $i \geq 0$ and therefore $\VL^0 \Ocup \CE \VL \supseteq \VL$. 
Thus, $\VL^0 \Ocup \CE \VL = \VL$.
\end{proof}

Note that in the general case, \ref{eq7.2} might be not a single solution satisfying a recursive grammar. 
It can, however, be shown to be a subset of any other existing solution.
This fact will not be used below and hence needs not be proven here, but we will refer to the solution \ref{eq7.2} as the \emph{minimal} one.

\begin{defn} \label{defn7.2}
For a given recursive grammar $\RG$, let $\Delta_{\RG}$ denote the set of $m$ formulas $D^1_t,... D^m_t$, where
\begin{enumerate}
\renewcommand{\labelenumi}{(\roman{enumi})}
\item $D^i_t \eqd (\CC^i_{s\a_i t} = L^i_{s\a_i t})$, if $\L_i$ is a lexicon represented by $L^i_{s\a_i t}$,
\item $D^i_t \eqd (\CC^i_{s\a_i t} = \CC^{j(i)}_{s\a_i t} \OR \CC^{k(i)}_{s\a_i t})$, if $\L_i = \L_{j(i)} \Ocup \L_{k(i)}$,
\item $D^i_t \eqd (\CC^i_{s\a_i t} = \CC^{j(i)}_{s\a_{j(i)} t} \cat \CC^{k(i)}_{s\a_{k(i)} t})$, if $\L_i = \L_{j(i)} \Ocat \L_{k(i)}$, so that $\a_i \eqd \a_{j(i)} \x \a_{k(i)}$,
\item $D^i_t \eqd (\CC^i_{s\a_i t} = \CC^{j(i)}_{s\a_{j(i)} t} \RA R^i_{\a_{j(i)}\a_i t})$, if $\L_i = \L_{j(i)} \RA \R_i$ and $\R_i$ is represented by $R^i_{\a_{j(i)}\a_i t}$,
\end{enumerate}
all the constants $\CC^1_{s\a_1 t}, ...\: \CC^m_{s\a_m t}$ are assumed to exist only in an extension $Ty_n^{\A+}$, but not in $Ty_n^\A$, and the $Ty_n^\A$ operators of logical concatenation  and semantic rule application are defined as follows:
\begin{multline*}
(\cat_{(s \a t) (s \b t) s \a\b t}) \eqd \l x_{s \a t} \l y_{s \b t} \l x_s \l x_\a \l x_\b \exists y_s \exists z_s \\
(x_s = y_s \wcat z_s \AND x_{s\a t} \: y_s \: x_\a \AND y_{s\b t} \: z_s \: x_\b),
\end{multline*}
\[
(\RA_{(s\a t)(\a\b t) s\b t}) \eqd \l x_{s\a t} \l y_{\a\b t} \l x_s \l x_\b \exists x_\a (x_{s\a t} \: x_s \: x_\a \AND y_{\a\b t} \: x_\a \: x_\b).
\]
\end{defn}

Below we will also make use of the $Ty_n^\A$ operator of a functional semantic rule application:
\[
(\FA_{(s\a t)(\a\b)s\b t}) \eqd \l x_{s\a t} \l y_{\a\b} (x_{s\a t} \RA (\l x_\a(= y_{\a\b} \: x_\a))).
\]
Note that every singleton $(=A^1_s) \X (=A^1_\a)$ can be equivalently written down as $(=A^1_s) \FA A^1_\a$.

We now prove that every language $\L_i$ of the minimal solution of a recursive grammar $\RG$ is represented by $(\Delta_{\RG}, \CC^i_{s\a_i t})$. To do so, we introduce the following auxiliary

\begin{defn} \label{defn7.3}
An $\a$-language $\L$ is said to have a \emph{pseudo-canonic representation} if it is represented by $(\Delta, L_{s\a t})$ such that for any word $w \in \A^*$ there exists a finite (may be empty) set $\{A^1_\a, \:...\: A^l_\a\}$ of $Ty_n^\A$ terms such that
\[
\Delta \vdash L_{s\a t} \: T_A(w) = \bigvee_{i = 1}^{l} (= A^i_\a)
\]
and adding $\Delta$ to $Ty_n^{\A+}$ as new axioms forms a conservative extension of $Ty_n^\A$ (i.e. for any $Ty_n^\A$ formula $A_t$ \; $\vdash A_t$ \: iff \: $\Delta \vdash A_t$).
\end{defn}

Note that a canonic representation of an $\a$-lexicon is a particular case of a pseudo-canonic representation.

\begin{lem} \label{lem7.2}
If $\Delta$ is a consistent set of $Ty_n^{\A+}$ formulas and for any general model $M$ of $Ty_n^\A$ and an assignment of variables in that model there exists an extension $M^+$ of $M$ for $Ty_n^{\A+}$ in which $\Delta$ is satisfied by the same assignment of variables, then adding $\Delta$ to $Ty_n^{\A+}$ forms a conservative extension of $Ty_n^\A$.
\end{lem}
\begin{proof}
By the Completeness Theorem, if not $\vdash A_t$, then there exists a general $Ty_n^\A$ model $M$ and assignment of variables $a$ by which $A_t$ is not satisfied in $M$. 
If $M$ can be extended to a model $M^+$ for $Ty_n^{\A+}$ where $\Delta$ is still satisfied by $a$, then by the Soundness Theorem the derivation $\Delta \vdash A_t$ is impossible.
\end{proof}

\begin{lem} \label{lem7.3}
The set $\{A^1_\a, \:...\: A^l_\a\}$ appearing in the first condition of Definition \ref{defn7.3} generates all meanings of $w$ in the language $\L$.
\end{lem}
\begin{proof}
If $A_a$ is a $Ty_n^\A$ term and $(w, A_\a) \in \L$, then $\Delta \vdash L_{s \a t} \: T_A(w) \: A_a$ and therefore
\[
\Delta \vdash (A_\a = A^1_\a) \:...\: \OR (A_\a = A^l_\a),
\]
which implies that
\[
\vdash (A_\a = A^1_\a) \:...\: \OR (A_\a = A^l_\a).
\]
Thus, by Lemma \ref{lem6.3}, any meaning of $w$ belongs to the logical closure of $\{A^1_\a, \:...\: A^l_\a\}$.
\end{proof}

\begin{lem} \label{lem7.4}
In any recursive grammar $\RG$, a lexicon $\L_i$ with a canonic representation $L^i_{s\a_i t}$ also has the pseudo-canonic representation $(\Delta_{\RG}, L^i_{s\a_i t})$.
\end{lem}
\begin{proof}
The first condition of Definition \ref{defn7.3} follows directly from the assumption that $L^i_{s\a_i t}$ is a canonic representation.

Now we prove that the set $\Delta_{\RG}$ is consistent and also satisfied by a given assignment of variables $a$ in some extension $M^+$ of an arbitrary $Ty_n^\A$ model $M$.
Lemma \ref{lem5.2} and the fact that a recursive grammar is satisfied by some logically closed languages $\L_1,...\:\L_m$ allow the extension $M^+$ to be defined as follows: let the function $\llb\CC^i_{s\a_i t}\rrb_{M^+}$ for given arguments $w \in \D_s$ and $v \in \D_{\a_i}$ take the true value iff $w \in \A^*$ and the language $\L_i$ contains pairs $(w, A_{\a_i})$ such that $\llb A_{\a_i}\rrb_{M^+,a} = v$. 
The conditions (i)-(iv) of Definition \ref{defn7.2} then imply that every $D^i_t$ is satisfied by $a$ in such an extended model. By Lemma \ref{lem7.2}, this proves the second condition of Definition \ref{defn7.3}.
\end{proof}

\begin{lem} \label{lem7.5}
If $\OM$ and $\ON$ are logical closures of $\M \in \T_\a$ and $\N \in \T_\a$ respectively, then $\overline{\OM \cup \ON} = \overline{\M \cup \N}$.
\end{lem}
\begin{proof}
Every element of $\overline{\M \cup \N}$ belongs to $\overline{\OM \cup \ON}$ because $\M \subseteq \OM$, $\N \subseteq \ON$, so a proof tree for any term from $\overline{\M \cup \N}$ is also a proof tree for $\overline{\OM \cup \ON}$. On the other hand, as every leaf of a proof tree for an element of $\overline{\OM \cup \ON}$ belongs to either $\OM$ or $\ON$, it can be expanded to a proof tree for the corresponding set and thus the entire tree can be converted to a proof tree for $\overline{\M \cup \N}$ with the same root, that proves its membership in $\overline{\M \cup \N}$.
\end{proof}

\begin{prop} \label{prop7.6}
Let $\a$-languages $\L$ and $\K$ have pseudo-canonic representations $(\Delta, L_{s \a t})$ and $(\Delta, K_{s \a t})$ respectively. 
Then their logical join $\L \Ocup \K$ is represented by $(\Delta, L_{s \a t} \OR K_{s \a t})$, which is also a pseudo-canonic representation.
\end{prop}
\begin{proof}
Let
\[
\Delta \vdash L_{s\a t} \: T_A(w) = \bigvee_{i = 1}^{l} (= A^i_\a) \qquad \text{and} \qquad
\Delta \vdash K_{s\a t} \: T_A(w) = \bigvee_{i = 1}^{k} (= B^i_\a).
\]
Then 
\[
\Delta \vdash (L_{s \a t} \OR K_{s \a t}) \: T_A(w) = \bigvee_{i = 1}^{l} (= A^i_\a) \OR \bigvee_{i = 1}^{k} (= B^i_\a).
\]
This proves that $(\Delta, L_{s \a t} \OR K_{s \a t})$ is a pseudo-canonic representation and therefore for any  $Ty_n^\A$ term $A_\a$,
\[
\Delta \vdash (L_{s \a t} \OR K_{s \a t}) \: T_A(w) \: A_\a \quad \Longleftrightarrow \quad 
\vdash \bigvee_{i = 1}^{l} (A_\a = A^i_\a) \OR \bigvee_{i = 1}^{k} (A_\a = B^i_\a).
\]
By Lemmas \ref{lem7.3}, \ref{lem6.3} and \ref{lem7.5} this proves that $(\Delta, L_{s \a t} \OR K_{s \a t})$ represents $\L \Ocup \K$.
\end{proof}

\begin{prop} \label{prop7.7}
Let an $\a$-language $\L$ have a pseudo-canonic representation $(\Delta, L_{s \a t})$ and a $\b$-language $\K$ have a pseudo-canonic representation $(\Delta, K_{s \b t})$. 
Then the logical concatenation $\L \Ocat \K$ is represented by \\
$(\Delta, L_{s \a t} \cat K_{s \b t})$, which is also a pseudo-canonic representation.
\end{prop}
\begin{proof}
Let \;$w = h_0 + t_0 = \:...\: = h_l + t_l$\; be all possible splits of a word $w$ into head and tail and
\[
\Delta \vdash L_{s\a t} \: T_A(h_m) = \bigvee_{i = 1}^{l_m} (= A^{m,i}_\a), \qquad
\Delta \vdash K_{s\b t} \: T_A(t_m) = \bigvee_{j = 1}^{k_m} (= B^{m,j}_\b)
\]
(for m = 0, ... l). Then by axioms \ref{eq4.1} - \ref{eq4.3} and metatheorems 
\begin{equation} \label{eq7.5}
\vdash \exists x_\a (A_t \OR B_t) = \exists x_\a A_t \OR \exists x_\a B_t,
\end{equation}
\begin{equation} \label{eq7.6}
\vdash \exists x_\g (x_\g = A_\g \: \AND \: B_t(x_\g)) = B_t(A_\g)
\end{equation}
(where $x_\g$ is not free in $A_\g$) we have:
\[
\Delta \vdash (L_{s \a t} \cat K_{s \b t}) \: T_A(w) = \bigvee_{m=0}^{l} \bigvee_{i = 1}^{l_m} \bigvee_{j = 1}^{k_m} (= A^{m,i}_\a) \X (= B^{m,j}_\b).
\]
This proves that $(\Delta, L_{s \a t} \cat K_{s \b t})$ is a pseudo-canonic representation and therefore for any $Ty_n^\A$ terms $A_\a$, $B_\b$,
\[
\Delta \vdash (L_{s \a t} \cat K_{s \b t}) \: T_A(w) \: A_\a B_\b \; \Longleftrightarrow \;
\vdash \bigvee_{m=0}^{l} \bigvee_{i = 1}^{l_m} \bigvee_{j = 1}^{k_m} (A_\a = A^{m,i}_\a) \AND (B_\b = B^{m,j}_\b).
\]
By Definition \ref{defn3.7} and Lemma \ref{lem6.3} this proves that $(\Delta, L_{s \a t} \cat K_{s \b t})$ represents $\L \Ocat \K$.
\end{proof}

\begin{prop} \label{prop7.8}
Let an $\a$-language $\L$ have a pseudo-canonic representation $(\Delta, L_{s \a t})$ and a rule $\R \subset \T_\a \otimes \T_\b$ have a canonic representation $R_{\a\b t}$ in the range of $\L$. Then the language $\L \RA \R$ is represented by $(\Delta, L_{s \a t} \RA \R_{\a\b t})$, which is also a pseudo-canonic representation.
\end{prop}
\begin{proof}
Let 
\[
\Delta \vdash L_{s\a t} \: T_A(w) = \bigvee_{i = 1}^{l} (= A^i_\a)
\]
and
\[
\vdash R_{\a\b t} \: A^i_\a = \bigvee_{j = 1}^{k_i} (= B^{i,j}_\b)
\]
(for $i = 1, ... l$). Then by metatheorem \ref{eq7.6} we have
\[
\Delta \vdash (L_{s\a t} \RA \R_{\a\b t}) \: T_A(w) = \bigvee_{i = 1}^{l} \bigvee_{j = 1}^{k_i} (= B^{i,j}_\b).
\]
This proves that $(\Delta, L_{s \a t} \RA \R_{\a\b t})$ is a pseudo-canonic representation and therefore for any $Ty_n^\A$ term $B_\b$,
\[
\Delta \vdash  (L_{s\a t} \RA \R_{\a\b t}) \: T_A(w) \: B_\b \quad \Longleftrightarrow \quad
\vdash \bigvee_{i = 1}^{l} \bigvee_{j = 1}^{k_i} (B_\b = B^{i,j}_\b).
\]
By Definition \ref{defn3.3} and Lemma \ref{lem6.3} this proves that $(\Delta, L_{s \a t} \RA \R_{\a\b t})$ represents $\L \RA \R$.
\end{proof}

Propositions \ref{prop7.6} - \ref{prop7.8} establish that, given some logically closed languages known to have pseudo-canonic representations that share a single common axiom set $\Delta$ and semantic rules that have canonic representations, one can build $Ty_n$ representations of their logical joins, logical concatenations and applications of the rules by means of the three corresponding $Ty_n^\A$ operators: $(\OR_{(s \a t) (s \a t) s \a t})$, $(\cat_{(s \a t) (s \b t) s \a\b t})$ and $(\RA_{(s\a t)(\a\b t) s\b t})$; moreover these representations are also pseudo-canonic and share the same axiom set $\Delta$.

\begin{lem} \label{lem7.9}
$Ty_n^\A$ operators of logical join, logical concatenation and semantic rule application reproduce the associativity and monotonicity properties:

\begin{eqnarray*}
&\vdash& (x_{s\a t} \OR y_{s\a t}) \cat z_{s\b t} \;=\; x_{s\a t} \cat z_{s\b t} \;\OR\; y_{s\a t} \cat z_{s\b t}, \\
&\vdash& z_{s\b t} \cat (x_{s\a t} \OR y_{s\a t}) \;=\; z_{s\b t} \cat x_{s\a t} \;\OR\; z_{s\b t} \cat y_{s\a t}, \\
&\vdash& (x_{s\a t} \OR y_{s\a t}) \RA z_{\a\b t} \;=\; x_{s\a t} \RA z_{\a\b t} \;\OR\; y_{s\a t} \RA z_{\a\b t}, \\
&\vdash& (x_{s\a t} \IMP y_{s\a t}) \;\IMP\; (x_{s\a t} \cat z_{s\b t} \IMP y_{s\a t} \cat z_{s\b t}), \\
&\vdash& (x_{s\a t} \IMP y_{s\a t}) \;\IMP\; (z_{s\b t} \cat x_{s\a t} \IMP z_{s\b t} \cat y_{s\a t}), \\
&\vdash& (x_{s\a t} \IMP y_{s\a t}) \;\IMP\; (x_{s\a t} \RA z_{\a\b t} \IMP y_{s\a t} \RA z_{\a\b t}).
\end{eqnarray*}
\end{lem}
\begin{proof}
All these properties follow from the definition of operator $(\cat_{(s \a t) (s \b t) s \a\b t})$ due to the tautologies
\[
\vdash (A_t \OR B_t) \AND C_t \;=\; A_t \AND C_t \OR B_t \AND C_t, \quad
\vdash C_t \AND (A_t \OR B_t) \;=\; C_t \AND A_t \OR C_t \AND B_t,
\]\[
\vdash (A_t \IMP B_t) \;\IMP\; (A_t \AND C_t \IMP B_t \AND C_t), \quad
\vdash (A_t \IMP B_t) \;\IMP\; (C_t \AND A_t \IMP C_t \AND B_t)
\]
and the metatheorems \ref{eq7.5} and
\[
\vdash \forall x_\g (A_t \IMP B_t) \;\IMP\; (\exists x_\g A_t \IMP \exists x_\g B_t).
\]
\end{proof}

\begin{thm} \label{thm7.10}
Every language $\L_i$ of the minimal solution of a recursive grammar $\RG$ is represented by $(\Delta_{\RG}, \CC^i_{s\a_i t})$.
\end{thm}
\begin{proof}
If (using the notations introduced in the proof of Proposition \ref{prop7.1}) $\ee \in \VL$, then there exists an index $i$ such that $\ee \in \VL^i$. 
As every $\VL^i$ is built from lexicons $\VL^0$ and $\VL^{(i-1)}$ by applying logical joins, logical concatenations and semantic rule applications, from Propositions \ref{prop7.6} - \ref{prop7.8} it follows (by induction on $i$) that every component $\L^i_j$ of $\VL^i$ has a pseudo-canonic representation $(\Delta_{\RG}, L^{i,j}_{\t_j})$, where for any $i > 0$
\[
L^{i,j}_{\t_j} \eqd L^{0,j}_{\t_j} \OR E^j_{\t_1...\t_m\t_j} \: L^{i-1,1}_{\t_1} \:...\: L^{i-1,m}_{\t_m},
\]
$\t_j \eqd s\a_j t$ and terms $E^j_{\t_1...\t_m\t_j}$ stand for $Ty_n^\A$ operators corresponding to the operator $\CE$. 
On the other hand, from Definition \ref{defn7.2} we have
\[
\Delta_{\RG} \vdash \CC^j_{\t_j} = L^{0,j}_{\t_j} \OR E^j_{\t_1...\t_m\t_j} \: \CC^1_{\t_1} \:...\: \CC^m_{\t_m}.
\]
As by Lemma \ref{lem7.9} the terms $E^j_{\t_1...\t_m\t_j}$ reproduce the monotonicity property of the operator $\CE$, it follows from this that for all $i \geq 0$
\[
\Delta_{\RG} \vdash L^{i,j}_{\t_j} \IMP \CC^j_{\t_j}
\]
and therefore, for every $j$ that
\[
\Delta_{\RG} \vdash \CC^j_{\t_j} \: T_A(w_j) \: A_{\a_j}.
\]

To prove that, inversely, the above derivations imply $\ee \in \VL^i$, consider the $Ty_n^{\A+}$ model defined in the proof of Lemma \ref{lem7.4}. As such a model can extend an arbitrary $Ty_n^\A$ model, these derivations imply that there exist such terms $A'_{\a_j}$ that $\ee' \eqd ((w_1, A'_{\a_1}),\:...\:(w_m, A'_{\a_m})) \in \VL$ and $\vdash A'_{\a_j} = A_{\a_j}$. As $\L_j$ are logically closed, this implies that $\ee \in \VL$ also.
\end{proof}

\begin{cor} \label{cor7.11}
Every language $\L_i$ of the minimal solution of a recursive grammar $\RG$ is represented by the term $\l x_s \l x_{\a_i} D_t \:\IMP\: \CC^i_{s\a_i t}$, where $D_t$ is the conjunction of all formulas from $\Delta_{\RG}$.
\end{cor}
\begin{proof}
See the note to Definition \ref{defn4.1}.
\end{proof}

The significance of these results is that they, in particular, establish a class of non-trivial $Ty_n$-representable languages which turns out to be decidable, in spite of the fact that $Ty_n$ is undecidable and hence so are generally $Ty_n$-representable languages. Indeed, assume all the semantic rules participating in a grammar $\RG$ to be non-degenerate. It is easy to see that in this assumption the domain of every language $\L_i$ is formed by a context-free grammar with the set of terminals being the union of domains of all the lexicons participating in $\RG$ and non-terminal symbols corresponding all the $\L_i$ languages. Therefore, this domain falls into the class of context-free (and hence - recursive) languages of the Chomsky hierarchy. Then, in the additional assumption of decidability of all the $\RG$ semantic rules (which holds, for example, for any functional rules), the problem of finding all the logically-independent representatives of the meanings of an $\L_i$ expression also turns out to be decidable, merely by attributing the terminal (leaf) nodes of its syntactic parsing tree with their (lexicon-determined) meanings and inheriting and evaluating the meanings of upper nodes according to the rules, from bottom up to the root node. Furthermore, though discarding the assumption of non-degenerateness of all the $\RG$ semantic rules may generally lead to languages $\L_i$ with domains being context-sensitive (in the classical sense), retaining only the assumption of the rules decidability still preserves the conclusion about decidability of languages $\L_i$. In this assumption languages $\L_i$ can be parsed, for example, by the two-staged process like this:

\begin{enumerate}
\item parse according to the context-free grammar $\overline{\RG}$ obtained from $\RG$ by replacing every rule $\R_i$ by the trivial non-degenerate rule with the canonic representation $\overline{R^i}_{\a_{j(i)}\a_i t} \eqd \l x_{\a_{j(i)}} \l x_{\a_i} \TT$
\item attribute the terminal (leaf) nodes of the parsing tree with their (lexicon-determined) meanings and inherit and evaluate the meanings of upper nodes according to the actual rules $\R_i$, or discard the parse if at any node the actual rule degenerates to $\l x_{\a_i} \FF$ when applied to all incoming meanings.
\end{enumerate}

To illustrate the technique of a recursive grammar representation, let us build a very simple English-like grammar from tiny lexicons, consisting of a few countable nouns and transitive verbs, the irregular verb "be", a proper noun and two determiners:
\begin{eqnarray*}
\Noun_{s(et)t} &=&(= /\text{builder}/) \FA \Builder_{et} \;\OR\\
               & &(= /\text{house}/) \FA \House_{et} \;\OR\\
	       & &(= /\text{car}/) \FA \Car_{et},\\
\Verbt_{s(eet)t} &=&(= /\text{build}/) \FA \Build_{eet} \;\OR\\
                 &=&(= /\text{sell}/) \FA \Sell_{eet},\\
\VerbBe_{s(eet)t} &=&(= /\text{be}/) \FA (=_{eet}),\\
\Name_{set}    &=&(= /\text{Jack}/) \FA \Jack_e,\\
\Det_{s((et)(et)t)t} &=&(= /\text{a}/) \FA \l y_{et} \l z_{et} \exists x_e (y \: x \;\AND\; z \: x) \;\OR\\
                 & &(= /\text{every}/) \FA (\IMP_{(et)(et)t}).
\end{eqnarray*}
The noun phrase can be defined by the axiom
\begin{eqnarray*}
\NP_{s((et)t)t} &=& \Det \pcat \Noun \FA \II_{((et)(et)t)(et)(et)t} \;\OR\\
		& & \Name \FA \l x_e \l y_{et} (y \; x)
\end{eqnarray*}
and the simple verb phrase (for the object of the third person in the singular number) by the axioms
\begin{eqnarray*}
\VPs_{s(et)t} &=& (\Verbt \cat (=/\text{s}/) \;\OR\; ((=/\text{is}/) \FA /\text{be}/) \RA \VerbBe) \pcat \NP \FA\\
	      & & \l z_{eet} \l y_{(et)t} \l x^2_e(y \: \l x^1(z \: x^1 \: x^2))\;\OR\\
	      & & ((=/\text{does not}/) \pcat \Verbt \;\OR\; ((= /\text{is not}/) \FA /\text{be}/) \RA \VerbBe) \pcat \NP \FA\\
	      & & \l z_{eet} \l y_{(et)t} \l x^2_e(y \: \l x^1(\NOT (z \: x^1 \: x^2))),\\
\end{eqnarray*}
where
\[
(\pcat) \eqd \l x_{s\a t} \l y_{s\b t} (x \cat (= /\text{ }/) \cat y)
\]
represents phrase concatenation (with a separating space).  
Note that the constant $\VerbBe$ occurs here as a semantic rule applied to simple $s$-languages which express the morphology of the verb "be". As these languages, as well as the language represented by $\VerbBe$ itself are finite lexicons, the results of the applications are lexicons too and hence conditions of the theorem \ref{thm7.10} still hold for the language we are defining.
The compound verb phrase can now be defined by the axioms like
\begin{eqnarray*}
\VPm_{s(et)t} &=& \VPs \pcat (= /\text{and}/) \pcat \VPs \FA (\AND_{(et)(et)et)}) \;\OR\\
              & &  \VPs \cat (= /\text{,}/) \pcat \VPm \FA (\AND_{(et)(et)et}),\\
\VP_{s(et)t} &=& \VPs \;\OR\; \VPm
\end{eqnarray*}
and the clause -- by the axiom
\[
\Cl_{stt} = \NP \pcat \VP \FA \II_{((et)t)(et)t}.
\]
Finally, the declarative sentence can be defined by the axiom
\[
\St_{stt} = \Cl \cat (=/\text{.}/) \;\OR\; (\Cl \cat (= /\text{,}/) \pcat \St \;\OR\;
                                            \Cl \pcat (= /\text{and}/) \pcat \St) \FA (\AND_{ttt}).
\]
If $\Delta_S$ is the set of the above axioms, then, by propositions \ref{prop7.6} - \ref{prop7.8}, we will have, for example:
\begin{eqnarray*}
\Delta_S &\vdash& \St \; /\text{Jack builds a house.}/ \; \exists x_e (\Build \: x_e \: \Jack \;\AND\; \House \: x_e),\\
\Delta_S &\vdash& \St \; /\text{Jack sells a car and builds a house.}/ \; (\exists x_e (\Sell \: x_e \: \Jack \;\AND\; \Car \: x_e) \AND \\
& & \exists x_e (\Build \: x_e \: \Jack \;\AND\; \House \: x_e)),\\
\Delta_S &\vdash& \St \; /\text{Jack does not build a house.}/ \; \NOT \exists x_e (\Build \: x_e \: \Jack \;\AND\; \House \: x_e),\\
\Delta_S &\vdash& \St \; /\text{every builder builds a house.}/ \; (\Builder \: \IMP \: \l x^2_e \exists x^1_e (\Build \: x^1_e \: x^2_e \;\AND\; \House \: x^1_e),\\
\Delta_S &\vdash& \St \; /\text{Jack is a builder, Jack builds a house.}/ \; (\exists x_e (\Builder \: x_e \;\AND\; \Jack = x_e) \AND \\
& & \exists x_e (\Build \: x_e \: \Jack \;\AND\; \House \: x_e)).
\end{eqnarray*}
Of course, along with these, the language represented by $(\Delta_S, \St)$ also admits semantically infelicitous sentences like, for example, "a house builds a builder" or "Jack is a house". If one would like to make such sentences ungrammatical, i.e. exclude them from the domain of the language represented by $(\Delta_S, \St)$, this can be achieved by applying more restrictive relational semantic rules, than those simple functional rules used in the above axioms; such rules must account the animacy categories of nouns, as well as categories of subjects and objects applicable to a verb. Another, less elegant, though practically may be simpler and more efficient solution is to split lexicons and derived languages to semantically homogenous components and allow concatenation of only matching pairs, like, for example:
\[
\Cl_{stt} = (\NPa \pcat \VPa \;\OR\; \NPi \pcat \VPi) \FA \II_{((et)t)(et)t},
\]
where $\NPa, \NPi$ are to represent animate and inanimate noun phrases and $\VPa, \VPi$ to represent verb phrases applicable to animate and inanimate subjects.
Note that this approach leads to splitting the verb lexicon to components corresponding different \emph{generic sentence frames} as defined in \cite{Fellbaum1998}; in particular, the verbs "build" and "sell", in terms of WordNet, will correspond the frame "Somebody ----s something" and the verb "be" will be shared by the frames "Something ----s something" and "Somebody ----s somebody".

\section{Context type, instructive and \\ context-dependent languages}

Introduction of an additional sort of individuals in $Ty_2$ of \cite{gallin:1975} allowed to accurately interpret in this theory the intensional logic ($IL$), that implied interpretation of the corresponding type ($s$ in Gallin's notations, not to be confused with the symbolic type of this work) as a set of possible ``states of world'' or ``states of affairs''. 
Similarly to this, $Ty_n^{\A}$ with $n > 2$ admits the same interpretation of one of its primitive types (let it be, for certainty, $c \eqd e_{n-1}$) and thus turns out to be capable of representing intensional languages, i.e. languages with intensional semantics, as well. 
In combination with the capabilities provided by the symbolic type, however, a more specific interpretation of type $c$ in $Ty_n^{\A}$ allows to go far beyond this straightforward generalization and achieve also a natural and effective account of those aspects of a language which are usually referenced in literature as dynamic semantics \cite{Kamp:TSR,Heim:1983-2,Seuren85,Groenendijk89dynamicmontague,Groenendijk90dynamicpredicate}.

The basic idea is to extend $Ty_n^{\A}$ with non-logical axioms like
\begin{equation} \label{eq9.1}
\l x_\a (\Deref_{\a c\a} \; S_\a \; (P_{cc} \; (\Setref_{\a\a c c} \; S_\a \; x_\a \; y_c))) =  \II_{\a\a}
\end{equation}
\begin{equation} \label{eq9.2}
\Unset_{\a cc} \; S_\a \super P_{cc} \super \Setref_{\a\a c c} \; S_\a \; x_\a = \II_{cc},
\end{equation}
where there could be some restrictions on the structure of terms $S_\a$ and $P_{cc}$ and
\[
(\super_{(\b\g)(\a\b)\a\g}) \eqd \l z_{\b\g} \l q_{\a\b} \l x_\a (z_{\b\g} \: (q_{\a\b} \: x_\a)),
\]
that allow to interpret type $c$ as a "store" \cite{sabry:1995} or, in the terminology of programming languages, "state" and transform a $Ty_n$ representation of a language to the state-passing style (SPS) \cite{wadler:1992,sabry:1995,jones:2003}. In such a transformation the state actually plays a role of a context, passing "side effects" of parsing some constituents of an expression to others, as we explore in details below. This justifies our referring to the type $c$ as to \emph{context} type (and motivates the notation).

Note that $cc$-terms like $\Setref_{\a\a c c} \; S_\a \; V_\a$ and $\Unset_{\a cc} \; S_\a$ present \emph{instructions} to "modify" a context, in the sense defined by the axioms \ref{eq9.1}, \ref{eq9.2}.  Proceeding from this observation, one can further assume certain $cc$-terms to denote instructions whose semantics might not be representable in $Ty_n^\A$, but still could be implemented in a computing or another system. Instructions to add or remove a non-logical axiom to or from a set of axioms associated with a context, or verify whether a given formula is provable from this set are important examples of such ``meta''-instructions. Although their semantics are not representable within $Ty_n$, they still might be denoted by $cc$ terms, like, say, $\Assert_{tcc} \; A_t$, $\Refute_{tcc} \; A_t$ or $\Test_{tcc} \; A_t$, and there could be a computing system capable of evaluating them in some conditions and providing an appropriate feedback in the event of failure. Other unrepresentable instructions might force a system to exchange information with an external environment via input/output devices or even to execute certain physical actions (like robotic ones).

It is easy to notice that associating a context with a set of non-logical axioms makes our $c$ type conceptually similar to the \emph{World} type of \cite{pollard:2005:lacl}. The principal technical (formal) difference is however that the \emph{World} type in the hyperintensional semantics is defined as a \emph{derived} subtype of the type of functions from \emph{Prop} to \emph{Bool}, while \emph{Prop} is introduced as a primitive type. Unlike that, in our model, vice-versa, the context type $c$ is introduced as primitive and the propositional type can be defined as the functional type $ct$ (see \cite{gallin:1975}). A fundamental theoretical benefit of the hyperintensional semantics is that it admits internalizing (i.e. making representable within the HOL) the entailment relation between propositions, along with the property of a set of propositions to be an ultrafilter, i.e. a maximum consistent set of propositions. However, as the \emph{World} type is defined just by the ultrafilterhood predicate, i.e., informally saying, assumed to ``settle every issue'' \cite{pollard:2005:lacl}, it is not obvious whether this theoretical benefit would imply practical benefits in an implementation of this formalism. At the same time, our context type, of course, enforces no assumptions about completeness (and even about consistency) of axioms associated with a context, that seems to be more pragmatic and realistic.

To summarize this introductory informal discussion, we can state that, while type $s$ of \cite{gallin:1975} and type \emph{World} of \cite{pollard:2005:lacl} stand for ``states of world'', the similar to them type $c$ of our model stands rather for ``states of \emph{an individual mind}'', which mind is not necessarily complete nor consistent, but, on one hand, still determines interpretation of a text and, on the other hand, is changeable (in particular, extendable) in the process of interpretation.

Moving on formalism, we first note that a superposition like
\[
P^l_{cc} \super P^{l-1}_{cc} \:...\: \super P^1_{cc},
\]
represents a sequence of instructions, executed in order from right to left, that is a \emph{program}, while the identity $\II_{cc}$ is the empty sequence, or "do-nothing" instruction.
(Of course an arbitrary $cc$-term would not invariably represent an executable instruction, but a real implementation should necessarily provide some sort of exception handling, reducing a non-executable instruction to another or possibly do nothing instruction.)

Thus, expressions of a $cc$-language mean some (either executable or non-executable) instructions and therefore a text, i.e.\ an ordered sequence of such expressions can be treated as a corresponding program. 
A few obvious features of such an \emph{instructive} language make it different in principle from a $t$- and even a $ct$-language:
\begin{enumerate}
\item an instructive language text meaning captures the order of statements, that is, its semantic interpretation can depend on that order;
\item it can also capture contradictory or equivalent truth meanings nested in subsequent instructions;
\item it provides a natural way to distinctly represent declarative, interrogative and imperative semantics.
\end{enumerate}

Given an initial context $C_c$, an instructive language text meaning $P_{cc}$ determines a new context $P_{cc} \: C_c$ depending on both the initial context $C_c$ and the program $P_{cc}$ which changes it, in the sense that a continuation of the text would already be applied to the new context. 
However, the program $P_{cc}$ itself here does not depend on the initial context. 
Our next step is to employ the context type for building context-dependent languages.

\begin{defn} \label{defn9.1}
A \emph{context-dependent $\a$-language} is a $c(c\x\a)$-language.
\end{defn}

Given a context-dependent $\a$-language $\L$ and a context term $C_c$, the operation of \emph{context application}
\[
\L \app C_c \;\eqd\; \L \RA \C\A(C_c),
\]
where the rule $\C\A(C_c) \subset \T_{c(c\x\a)} \otimes \T_{c\x\a}$ is represented by
\[
CA_{(cc)(c\a)c\a t}(C_c) \eqd \l x_{cc} \l y_{c\a} ((= x_{cc} \: C_c) \X (= y_{c\a} \: C_c)),
\]
and the operation of \emph{context instantiation}
\[
\L \dn C_c \;\eqd\; \L \RA \C\I(C_c),
\]
where the rule $\C\I(C_c) \subset \T_{c(c\x\a)} \otimes \T_\a$ is represented by
\[
CI_{(cc)(c\a)\a t}(C_c) \eqd \l x_{cc} \l y_{c\a} (= y_{c\a} \: C_c),
\]
produce a $c\x\a$-language and a regular $\a$-language, respectively, that depend on a context $C_c$. 
Inversely, given a regular $\a$-language and a term $F_{\a cc}$, the operation of \emph{context raising}
\[
\L \up F_{\a cc} \;\eqd\; \L \RA \C\R(F_{\a cc}),
\]
where the rule $\C\R(F_{\a cc}) \subset\T_\a  \otimes \T_{c(c\x\a)}$ is represented by
\[
CR_{\a(cc)(c\a)t}(F_{\a cc}) \eqd \l x_\a ((= F_{\a cc} \: x_\a) \X (= \l x_c x_\a)),
\]
produces a context-dependent $\a$-language. Note that for any language and context
\[
\vdash (\L \up \l x_\a \II_{cc}) \dn C_c \;=\; \L,
\]
that is, raising a regular language by $\l x_\a \II_{cc}$ converts it into a context-dependent language which however is essentially equivalent to the original one.

It is also noticeable that an instructive, that is a regular $cc$-language, turns out to be a singular case of a context-dependent unit-type-language.  
The reason why a context-dependent $\a$-language is required to be of type $c(c\x\a)$, rather than of a simpler type $c\a$, should become obvious from the following.
For sake of brevity we employ notation $\oa \eqd c(c\x\a)$ (for an arbitrary type $\a$) everywhere below.

\begin{defn} \label{defn9.2}
Let $\L$ be a context-dependent $\a$-language and $\K$ - a context-dependent $\b$-language. Then their \emph{anaphoric logical concatenation} is the context-dependent $\a\x\b$-language
\[
\L \acat \K \eqd \L \Ocat \K \RA \A\C,
\]
where the rule $\A\C \subset \T_{\oa\x\ob} \otimes \T_{\oab}$ is represented by the term
\[
AC_{\oa\ob\:\oab t} \eqd \l x^1_{cc} \l y^1_{c\a} \l x^2_{cc} \l y^2_{c\b} ((= x^2_{cc} \super x^1_{cc}) \X (= y^1_{c\a}) \X (= y^2_{c\b} \super x^1_{cc})),
\]
and their \emph{cataphoric logical concatenation} is the context-dependent $\a\x\b$-language
\[
\L \ccat \K \eqd \L \Ocat \K \RA \C\C,
\]
where the rule $\C\C \subset \T_{\oa\x\ob} \otimes \T_{\oab}$ is represented by the term
\[
CC_{\oa\ob\:\oab t} \eqd \l x^1_{cc} \l y^1_{c\a} \l x^2_{cc} \l y^2_{c\b} ((= x^1_{cc} \super x^2_{cc}) \X (= y^1_{c\a} \super x^2_{cc}) \X (= y^2_{c\b})).
\]
\end{defn}

For informal interpretation of this definition consider the superposition of an anaphoric concatenation with a context application:
\[
(\L \acat \K) \app C_c = (\L \Ocat \K) \RA (\C\A(C_c) \super \A\C),
\]
where the rule superposition $\C\A(C_c) \super \A\C$ is represented by the term
\[
\l x^1_{cc} \l y^1_{c\a} \l x^2_{cc} \l y^2_{c\b} ((= x^2_{cc} \: (x^1_{cc} \: C_c)) \X (= y^1_{c\a} \: C_c) \X (= y^2_{c\b} \: (x^1_{cc} \: C_c))).
\]
Thus, the anaphoric concatenation affects a meaning of the right-hand constituent by a context affected by the left-hand one, while the latter is affected directly by the initial context.
Symmetrically, for the cataphoric concatenation, the effective rule superposition is represented by the term
\[
\l x^1_{cc} \l y^1_{c\a} \l x^2_{cc} \l y^2_{c\b} ((= x^1_{cc} \: (x^2_{cc} \: C_c)) \X (= y^1_{c\a} \: (x^2_{cc} \: C_c)) \X (= y^2_{c\b} \: C_c)),
\]
that is, in this case an initial context is passed to the right-hand constituent and the affected context to the left-hand one.

Note that the construction of type $\oa \eqd c(c\x\a)$ coincides with the construction of the \emph{state monad} type in functional programming languages \cite{wadler:1992,jones:2003} and the context raising rule 
\[
CR_{\a(cc)(c\a)t}(\l x_\a\II_{cc}) \eqd \l x_\a ((= \II_{cc}) \X (= \l x_c x_\a))
\]
precisely corresponds to the \emph{return} operator of that monad. 
It is also easy to verify that the rules $AC_{\oa\ob\:\oab t}$ and $CC_{\oa\ob\:\oab t}$ defining operations $\acat$ and $\ccat$ come as results of \emph{binding} by the corresponding operator of the same monad, in the first case -- a meaning $(x^1_{cc}, y^1_{c\a})$ of the left constituent with the operation
\[
AK_{\a\:\oab} \eqd \l x_\a ((x^2_{cc}, \l y_c x_\a, y^2_{c\b})
\]
and in the second case -- a meaning $(x^2_{cc}, y^2_{c\b})$ of the second constituent with the symmetric operation
\[
CK_{\b\:\oab} \eqd \l x_\b ((x^1_{cc}, \l y_c x_\b, y^1_{c\a}),
\]
which operations merely combine a monadic-type-meaning of one of the constituents with a regular-type-meaning of the other. Of course, this note is nothing else than a rephrasing of the above informal interpretation of operations $\acat$ and $\ccat$, in terms of the state monad.

As the new operations $\acat, \ccat, \app, \up$ and $\dn$ are defined as rule applications to regular concatenations or just special rule applications, they also reveal associativity and monotonicity properties:

\begin{lem} \label{lem9.1}
For any context-dependent $\a$-languages $\K, \L$, a context-dependent $\b$-language $\M$, regular $\a$-languages $\K_0, \L_0$ and a context term $C_c$
\begin{eqnarray*}
(\K \Ocup \L) \acat \M \;&=&\; \K \acat \M \;\Ocup\; \L \acat \M, \\
\M \acat (\K \Ocup \L) \;&=&\; \M \acat \K \;\Ocup\; \M \acat \L, \\
\end{eqnarray*}
\begin{eqnarray*}
(\K \Ocup \L) \ccat \M \;&=&\; \K \ccat \M \;\Ocup\; \L \ccat \M, \\
\M \ccat (\K \Ocup \L) \;&=&\; \M \ccat \K \;\Ocup\; \M \ccat \L, \\
(\K_0 \Ocup \L_0) \up C_c \;&=&\; \K_0 \up C_c \;\Ocup\; \L_0 \up C_c, \\
(\K \Ocup \L) \dn C_c \;&=&\; \K \dn C_c \;\Ocup\; \L \dn C_c, \\
\K \subset \L \;&\Longrightarrow &\; \K \acat \M \subset \L \acat \M, \\
\K \subset \L \;&\Longrightarrow &\; \M \acat \K \subset \M \acat \L, \\
\K \subset \L \;&\Longrightarrow &\; \K \ccat \M \subset \L \ccat \M, \\
\K \subset \L \;&\Longrightarrow &\; \M \ccat \K \subset \M \ccat \L, \\
\K_0 \subset \L_0 \;&\Longrightarrow &\; \K_0 \up C_c \subset \L_0 \up C_c, \\
\K \subset \L \;&\Longrightarrow &\; \K_0 \dn C_c \subset \L_0 \dn C_c.
\end{eqnarray*}
\end{lem}

\begin{prop} \label{prop9.2}
If a context-dependent $\a$-language $\L$ has a pseudo-canonic representation $(\Delta, L_{s\oa t})$ and a context-dependent $\b$-language $\K$ has a pseudo-canonic representation $(\Delta, K_{s\ob t})$, then the concatenations $\L \acat \K$ and $\L \ccat \K$ have pseudo-canonic representations $(\Delta, L_{s\oa t} \acat K_{s\ob t})$ and $((\Delta, L_{s\oa t} \ccat K_{s\ob t})$ respectively, where
\begin{multline*}
(\acatt_{(s\oa t)(s\ob t)s\oab t}) \eqd \l z_{s\oa t} \l z_{s\ob t} \l x_s \l x_{cc} \l y_{c\a} \l y_{c\b} \exists x^1_s \exists x^2_s \exists x^1_{cc} \exists x^2_{cc} \exists y^2_{c\b} \\
(x_s = x^1_s + x^2_s \:\AND\: x_{cc} = x^2_{cc} \super x^1_{cc} \:\AND\: y_{c\b} = y^2_{c\b} \super x^1_{cc} \:\AND\: z_{s\oa t} \: x^1_s \: x^1_{cc} \: y_{c\a} \:\AND\: z_{s\ob t} \: x^2_s \: x^2_{cc} \: y^2_{c\b}),
\end{multline*}
\begin{multline*}
(\ccatt_{(s\oa t)(s\ob t)s\oab t})  \eqd \l z_{s\oa t} \l z_{s\ob t} \l x_s \l x_{cc} \l y_{c\a} \l y_{c\b} \exists x^1_s \exists x^2_s \exists x^1_{cc} \exists x^2_{cc} \exists y^1_{c\a} \\
(x_s = x^1_s + x^2_s \:\AND\: x_{cc} = x^1_{cc} \super x^2_{cc} \:\AND\: y_{c\a} = y^1_{c\a} \super x^2_{cc} \:\AND\: z_{s\oa t} \: x^1_s \: x^1_{cc} \: y^1_{c\a} \:\AND\: z_{s\ob t} \: x^2_s \: x^2_{cc} \: y_{c\b}).
\end{multline*}
\end{prop}
\begin{proof}
The proof follows from Propositions \ref{prop7.7}, \ref{prop7.8} and Definition \ref{defn9.2}, taking into account that the terms $AC_{\oa\ob\:\oab t}$ and $CC_{\oa\ob\:\oab t}$ are canonic representations of the corresponding rules (in any domain).
\end{proof}

\begin{prop} \label{prop9.3}
If a context-dependent $\a$-language $\L$ has a pseudo-canonic representation $(\Delta, L_{s\oa t})$, then a context instantiation $\L \dn C_c$ has a pseudo-canonic representation $(\Delta, L_{s\oa t} \dn C_c)$, where
\[
(\dn_{(s\oa t)cs\a t}) \eqd \l z_{(s\oa t)cs\a t} \l z_c \l x_s \l y_\a \exists x_{cc} \exists y_{c\a} (z_{s\oa t} \: x_s \: x_{cc} \: y_{c\a} \:\AND\: y_\a = y_{c\a} \: z_c).
\]
\end{prop}
\begin{proof}
The proof follows from Proposition \ref{prop7.6}, taking into account that the term $CA_{\oa c\a t}(C_c)$ is a canonic representation (in any domain) of the rule that defines the context instantiation operation.
\end{proof}

\begin{prop} \label{prop9.4}
If an $\a$-language $\L$ has a pseudo-canonic representation $(\Delta, L_{s\a t})$, than a context raising $\L \up F_{\a cc}$ has a pseudo-canonic representation $(\Delta, L_{s\a t} \up F_{\a cc})$, where
\begin{multline*}
(\up_{(s\a t)(\a cc)s\oa t}) \eqd \l z_{s\a t} \l z_{\a cc} \l x_s \l x_{cc} \l x_{c\a} \exists x_\a \\
(z_{s\a t} \: x_s \: x_\a \:\AND\: x_{cc} = z_{\a cc} \: x_\a \:\AND\: x_{c\a} = \l x_c x_\a).
\end{multline*}
\end{prop}
\begin{proof}
Similar to Proposition \ref{prop9.3}, with taking into account that the term $CR_{\a\oa t}$ is a canonic representation (in any domain) of the rule that defines the context raising operation.
\end{proof}

\begin{lem} \label{lem9.5}
$Ty_n^\A$ operators of logical join, anaphoric or cataphoric concatenations and context instantiation or raising reproduce the associativity and monotonicity properties:
\begin{eqnarray*}
&\vdash& (x_{s\oa t} \OR y_{s\oa t}) \acat z_{s\ob t} \;=\; x_{s\oa t} \acat z_{s\ob t} \;\OR\; y_{s\oa t} \acat z_{s\ob t}, \\
&\vdash& z_{s\ob t} \acat (x_{s\oa t} \OR y_{s\oa t}) \;=\; z_{s\ob t} \acat x_{s\oa t} \;\OR\; z_{s\ob t} \acat y_{s\oa t}, \\
&\vdash& (x_{s\oa t} \OR y_{s\oa t}) \ccat z_{s\ob t} \;=\; x_{s\oa t} \ccat z_{s\ob t} \;\OR\; y_{s\oa t} \ccat z_{s\ob t}, \\
&\vdash& z_{s\ob t} \ccat (x_{s\oa t} \OR y_{s\oa t}) \;=\; z_{s\ob t} \ccat x_{s\oa t} \;\OR\; z_{s\ob t} \ccat y_{s\oa t}, \\
&\vdash& (x_{s\a t} \OR y_{s\a t}) \up z_{\a cc} \;=\; x_{s\a t} \up z_{\a cc} \;\OR\; y_{s\a t} \up z_{\a cc}, \\
&\vdash& (x_{s\oa t} \OR y_{s\oa t}) \dn z_c \;=\; x_{s\oa t} \dn z_c \;\OR\; y_{s\oa t} \dn z_c, \\
&\vdash& (x_{s\oa t} \IMP y_{s\oa t}) \;\IMP\; (x_{s\oa t} \acat z_{s\ob t} \IMP y_{s\oa t} \acat z_{s\ob t}), \\
&\vdash& (x_{s\oa t} \IMP y_{s\oa t}) \;\IMP\; (z_{s\ob t} \acat x_{s\oa t} \IMP z_{s\ob t} \acat y_{s\oa t}), \\
&\vdash& (x_{s\oa t} \IMP y_{s\oa t}) \;\IMP\; (x_{s\oa t} \ccat z_{s\ob t} \IMP y_{s\oa t} \ccat z_{s\ob t}), \\
&\vdash& (x_{s\oa t} \IMP y_{s\oa t}) \;\IMP\; (z_{s\ob t} \ccat x_{s\oa t} \IMP z_{s\ob t} \ccat y_{s\oa t}), \\
&\vdash& (x_{s\a t} \IMP y_{s\a t}) \;\IMP\; (x_{s\a t} \up z_{\a cc} \IMP y_{s\a t} \up z_{\a cc}), \\
&\vdash& (x_{s\oa t} \IMP y_{s\oa t}) \;\IMP\; (x_{s\oa t} \dn z_c \IMP y_{s\oa t} \dn z_c).
\end{eqnarray*}
\end{lem}
\begin{proof}
The proof follows from the definitions of the operators and Lemma \ref{lem7.9}.
\end{proof}

Propositions \ref{prop9.2} -- \ref{prop9.4} and Lemmas \ref{lem9.1} and \ref{lem9.5} show that the main result of Section 7 -- Theorem \ref{thm7.10} -- can be generalized in a straightforward manner to the case of a recursive grammar that contains a mix of regular and context-dependent languages and relations involving all seven operations $\Ocup, \Ocat, \acat, \ccat, \RA, \up$ and $\dn$.

Let us see how the sample recursive grammar built in the previous section can be converted to an instructive (and hence context-dependent) language grammar and extended to acquire the two important features: distinction of declarative and interrogative semantics and pronoun dereferencing. To avoid excessive technical complexity and make the example more demonstrative, we will allow pronoun "he" to appear only on place of a clause subject, referencing the subject of another (previous) clause expressed by a name, and pronoun "it" only on place of an object, referencing the object of another (previous) affirmative verb phrase. The first can be achieved by defining the subject noun phrase by the axioms like
\begin{eqnarray*}
\SNP_{s\oett t} &=& (\Det \pcat \Noun \FA \II_{((et)(et)t)(et)(et)t}) \up \l x_{(et)t} \II_{cc} \;\OR\\
		& & (\Name \FA \l x_e \l y_{et} (y \; x)) \up \Setref \; \He_{(et)t}) \;\OR\\
		& & \SPron \FA (\II_{cc}, \:\Deref_{((et)t)c(et)t}),\\
\SPron_{(et)t}   &=& (= /\text{he}/) \FA \He_{(et)t}.
\end{eqnarray*}
Note that the constant $\He_{(et)t}$ here does not denote any specific meaning, but plays role of a symbol to assign a meaning and carry it to a corresponding reference.

Accurate treatment of the pronoun on place of an object (which we want to be capable of referencing an arbitrarily expressed object of another verb phrase) requires a more sophisticated technique. First, we specialize the type of the object noun phrase to $(eet)et$ and define it by the axioms
\begin{eqnarray} \label{eq9.3}
\ONP_{s\oeetet t} &=& (\NP \FA O) \up \l x_{(eet)et} \II_{cc} \;\OR\\
                  & & \OPron \FA (\II_{cc}, \Deref_{((eet)et)c(eet)et}),\nonumber\\
\OPron_{(eet)et}  &=& (= /\text{it}/) \FA \It_{(eet)et},\nonumber
\end{eqnarray}
where $\NP$ is defined the same way as in the previous section and
\[
O_{((et)t)(eet)et} \eqd \l n_{(et)t} \l v_{eet} \l y_e (n \; \l x_e (v \; x \; y))
\]
simply maps its $(et)t$ meaning to $(eet)et$. Note that these definitions do not yet provide setting the $\It$ reference. This is because in general case it should be dereferenced to a meaning determined not only by the object noun phrase itself, but also by a verb applied to it. For example, "it" referencing the object in the verb phrases "builds a house" and "sells a house" should be dereferenced to different meanings: "a house which is built" and "a house which is sold", respectively. Therefore, our $\It$ reference can be set only upon processing the entire verb phrase, as captured in the following axiom:
\begin{eqnarray} \label{eq9.4}
\VPs_{s\oet t} &=& \Verbt \cat (=/\text{s}/) \pcat \ONP \FA (SetIt, Apply) \;\OR\\
               & & ((=/\text{is}/) \FA /\text{be}/) \RA \VerbBe) \pcat \ONP \FA \nonumber\\
	       & & (\l v_{eet} \l x_{cc} \l y_{c(eet)et} \II_{cc}, Apply) \;\OR\nonumber\\
	       & & ((=/\text{does not}/) \pcat \Verbt \;\OR\; ((= /\text{is not}/) \FA /\text{be}/) \RA \VerbBe) \FA\nonumber\\
	       & & (\l v_{eet} \l x_{cc} \l y_{c(eet)et} \II_{cc}, \NOT Apply),\nonumber
\end{eqnarray}
where
\begin{eqnarray*}
SetIt &\eqd& \l v_{eet} \l x_{cc} \l y_{c(eet)et} \l z_c (\Setref \; \It \; \l u_{eet} (y \; (x \; z) \; (v \AND u))),\\
Apply  &\eqd& \l v_{eet} \l x_{cc} \l y_{c(eet)et} \l z_c (y \; (x \; z) \; v).
\end{eqnarray*}
We shall see below how interaction of axioms \ref{eq9.3} and \ref{eq9.4} makes the correct semantics, but so far have to complete our language definition.

The context-dependent versions of the compound declarative verb phrase, clause and sentence can be defined straightforward by the axioms
\begin{eqnarray*}
\VPm_{s\oet t} &=& \VPs \pcat (= /\text{and}/) \apcat \VPs \FA \r\;(\AND_{(et)(et)et)}) \;\OR\\
              & &  \VPs \cat (= /\text{,}/) \apcat \VPm \FA \r\;(\AND_{(et)(et)et}),\\
\VP_{s\oett t} &=& \VPs \;\OR\; \VPm,\\
\Cl_{s\ot t} &=& \SNP \apcat \VP \FA \r\; \II_{((et)t)(et)t},\\
\Sd_{s\ot t} &=& \Cl \cat (=/\text{.}/) \;\OR\; (\Cl \cat (= /\text{,}/) \apcat \Sd \;\OR\;
	      \Cl \pcat (= /\text{and}/) \apcat \Sd) \FA \r\; (\AND_{ttt}),
\end{eqnarray*}
where
\[
(\apcat) \eqd \l x_{s\oa t} \l y_{s\ob t} \; x \cat (= /\text{ }/) \acat y,
\]
and
\[
\r_{(\a\b)\oa\ob} \eqd \l z_{\a\b} \l y_{cc} \l y_{c\a}\: (y_{cc},\; \l x_c\: (z\;(y_{c\a}\;x))).
\]

Now, to add the simplest ("yes/no") interrogative sentence to our language, we need the axioms like
\begin{eqnarray*}
\IVPdo_{s\oet t} &=& \Verbt \pcat \ONP \FA (SetIt, Apply),\\
\IVPbe_{s\oet t} &=& \ONP \FA (\l x_{cc} \l y_{c(eet)et} \II_{cc}, Apply \; (=_{eet})),\\
\Cli_{s\ot t} &=& ((= /\text{does}/) \pcat \SNP \apcat \IVPdo \;\OR\; (= /\text{is}/) \pcat \SNP \apcat \IVPbe) \FA\\
              & & \r \; \II_{((et)t)(et)t},\\
\Si_{s\ot t} &=& \Cli \cat (= /\text{?}/).
\end{eqnarray*}
Finally, we wrap the declarative and interrogative sentences into the general sentence of the instructive type:
\begin{eqnarray} \label{eq9.5}
\St_{s(cc)t} &=& \Sd \FA \l y_{cc} \l y_{ct} (\Assert_{tcc} \super y_{ct} \super y_{cc}) \;\OR\\
	     & & \Si \FA \l y_{cc} \l y_{ct} (\Test_{tcc} \super y_{ct} \super y_{cc}).\nonumber
\end{eqnarray}

If now $\Delta_S$ denotes the new set of axioms, then, due to the results of this section and in the assumption that $Ty_n^\A$ has been extended by the axioms \ref{eq9.1}, \ref{eq9.2}, we will have, for example:
\begin{eqnarray*}
\Delta_S &\vdash& \St \; /\text{Jack is a builder.}/ \\ 
         && (\Assert \; \exists x_e (\Builder \: x_e \;\AND\; \Jack = x_e) \super \Setref \; \He \; \Jack),\\
\Delta_S &\vdash& \St \; /\text{is Jack a builder?}/ \\ 
	&& (\Test \; \exists x_e (\Builder \: x_e \;\AND\; \Jack = x_e) \super \Setref \; \He \; \Jack),\\
\Delta_S &\vdash& \St \; /\text{Jack is a builder, he builds a house.}/ \\ 
	&& (\Assert \; (\exists x_e (\Builder \: x_e \;\AND\; \Jack = x_e) \AND \\
	&& \exists x_e (\Build \: x_e \: \Jack \;\AND\; \House \: x_e)) \super \Setref \; \He \; \Jack),\\
\Delta_S &\vdash& \St \;  /\text{Jack builds a house and sells it.}/ \\ 
	&& (\Assert \; (\exists x_e (\Build \: x_e \: \Jack \; \;\AND\; \House \: x_e) \;\AND\;\\
	&& \exists x_e ((\Build \;\AND\; \Sell) \: x_e \: \Jack \;\AND\; \House \: x_e)) \super\\
	&& \Setref \; \It \; \l u_{eet} \l y_e \exists x_e ((\Build \;\AND\; u_{eet}) \: x_e \: y_e \;\AND\; \House \: x_e) \super\\
	&& \Setref \; \He \; \Jack).
\end{eqnarray*}
Due to the metatheorem
\[
\vdash (\exists x_e A \;\AND\; \exists x_e (A \;\AND\; B)) = \exists x_e (A \;\AND\; B),
\]
where $A$ and $B$ are arbitrary formulas, the last derivation also implies a shorter form
\begin{eqnarray*}
\Delta_S &\vdash& \St \;  /\text{Jack builds a house and sells it.}/ \\ 
	&& (\Assert \; (\exists x_e ((\Build \;\AND\; \Sell) \: x_e \: \Jack \;\AND\; \House \: x_e)) \super\\
	&& \Setref \; \It \; \l u_{eet} \l y_e \exists x_e ((\Build \;\AND\; u_{eet}) \: x_e \: y_e \;\AND\; \House \: x_e) \super\\
	&& \Setref \; \He \; \Jack).
\end{eqnarray*}
Similarly, we would have also
\begin{eqnarray*}
\Delta_S &\vdash& \St \;  /\text{Every builder builds a house and sells it.}/ \\ 
	&& (\Assert \; (\Builder \IMP \l y_e \exists x_e ((\Build \;\AND\; \Sell) \: x_e \: y_e \;\AND\; \House \: x_e)) \super\\
	&& \Setref \; \It \; \l u_{eet} \l y_e \exists x_e ((\Build \;\AND\; u_{eet}) \: x_e \: y_e \;\AND\; \House \: x_e)).
\end{eqnarray*}
The last examples give some hints on how this formalism could address donkey pronouns \cite{Geach:Reference,Kamp:TSR,Heim:PhD,LLBook18955}, remaining entirely within the mainstream higher-order logic semantics. Of course, elaboration of this topic deserves a separate publication.

Note that, due to the special semantic rules applied in the axiom \ref{eq9.5}, the sentence meaning captures not only an assertion or test instruction, but also all the "side effects" of its reading. This is essential in order to pass the side effects between sentences, if we would like to further define a language like
\[
\Text_{s(cc)t} = \St \;\OR\; \St \apcat \Text \FA (\super_{(cc)(cc)cc}),
\]
so that, if $\Delta_S$ contains this last axiom too, then we would also have
\begin{eqnarray*}
\Delta_S &\vdash& \Text \; /\text{Jack is a builder. he builds a house.}/\\
         && (\Assert \; \exists x_e (\Build \: x_e \: \Jack \;\AND\; \House \: x_e) \super\\
	 && \Setref \; \It \; \l u_{eet} \l y_e \exists x_e ((\Build \;\AND\; u_{eet}) \: x_e \: y_e \;\AND\; \House \: x_e) \super\\
         && \Assert \; \exists x_e (\Builder \: x_e \;\AND \; \Jack = x_e) \super\\
	 && \Setref \; \He \; \Jack).
\end{eqnarray*}

\section{Translation/expression operators and \\ partial translation}

Let an $\a$-language $\L$ be represented by a term $L_{s \a t}$ and let $w \in \A^*$ be a non-ambiguous valid expression of this language, so that
\[
\vdash \exists_1 x_\a (L_{s \a t} \: /w/ \: x_\a).
\]
Then, by the basic property of the description operator $\io_{(\a t)\a}$,
\[
\vdash \io_{(\a t)\a} (L_{s \a t} \: /w/) = A_\a,
\]
where $\{A_\a\}$ is any generator of the meaning of $w$ in $\L$ (for a product type $\a \eqd \b\x\g$ here and below $\io_{(\a t)\a}$ denotes the pair
\begin{equation} \label{eq10.1}
(\l z_{\b\g t} (\io_{(\b t)\b} \: \l x_\b \exists x_\g (z_{\b\g t} \: x_\b \: x_\g)), \; \l z_{\b\g t} (\io_{(\g t)\g} \: \l x_\g \exists x_\b (z_{\b\g t} \: x_\b \: x_\g)))
\end{equation}
which definition may be applied recursively).
Similarly, if $\{A_\a\}$ is a generator of a meaning which has a single expression $w$ in $\L$, then
\[
\vdash \io_{(st)s} \l x_s (L_{s \a t} \: x_s \: A_\a) = /w/.
\]
Thus, given an $\a$-language represented by a term $L_{s \a t}$, the terms
\[
\t_{(s \a t)s\a} \eqd \l x_{s \a t} \l x_s (\io_{(\a t)\a} \: (x_{s \a t} \: x_s))
\]
and
\[
\s_{(s \a t) \a s} \eqd \l x_{s \a t} \l x_\a (\io_{(st)s} \: \l x_s (x_{s \a t} \: x_s \: x_\a))
\]
enable $Ty_n$ representations of expression-to-meaning and meaning-to-expression mapping determined by the language, in forms like $\t_{(s \a t)s\a} \; L_{s \a t}$ and $\s_{(s \a t)\a s} \; L_{s \a t}$ respectively. They are referred to below as \emph{translation} and \emph{expression operators}.

Application of the translation operator to an invalid expression, of course, is not provable to be equal to any meaning.  However, it still represents a meaning that the expression might acquire upon possible completion of the language definition. Here is a precise formulation of this statement:

\begin{prop} \label{prop10.1}
Let $(\Delta, L_{s\a t})$ be a pseudo-canonic representation of a language $\L$ and 
\begin{equation} \label{eq10.2}
L^+_{s\a t} \:\eqd\: L_{s\a t} \OR \l x_s (= \t_{(s\a t)s\a} \: L_{s\a t} \: x_s).
\end{equation}
If $w$ is either a valid and non-ambiguous or invalid expression of $\L$, then
\[
\Delta \vdash \t_{(s\a t)s\a} \: L^+_{s\a t} \: /w/ = \t_{(s\a t)s\a} \: L_{s\a t} \: /w/.
\]
\end{prop}
\begin{proof}
If $w$ is a non-ambiguous valid expression of $\L$ and $A_\a$ its meaning, then
\[
\Delta \vdash L_{s\a t} \: /w/ = (= A_\a)
\]
and
\[
\Delta \vdash \io_{(\a t)\a} \: (L_{s\a t} \: /w/) = A_\a,
\]
so that 
\[
\Delta \vdash L^+_{s\a t} \: /w/ = L_{s\a t} \: /w/.
\]
Otherwise, if $w$ is an invalid expression, then by Lemma \ref{lem6.1}
\[
\Delta \vdash L_{s\a t} \: /w/ = \l x_\a \FF,
\]
so that
\[
\Delta \vdash L^+_{s\a t} \: /w/ = (= \t_{(s\a t)s\a} \: L_{s\a t} \: x_s)
\]
and therefore
\[
\Delta \vdash \io_{(\a t)\a} \: (L^+_{s\a t} \: /w/) = \t_{(s\a t)s\a} \: L_{s\a t} \: x_s.
\]
\end{proof}

Thus, the language defined by (\ref{eq10.2}) extends the language $\L$ with all and only those pairs $(w, \t_{(s\a t)s\a} \: L_{s\a t} \: /w/)$ in which $w$ is not a valid non-ambiguous expression of $\L$, while similar pairs in which $w$ is such an expression are fully contracted. 
In this sense, such a \emph{self-extension} of a language $\L$ does not add anything new to it. 
However, when a language is nested into another language of a recursive grammar, employment of its self-extension enables some new and useful features, as established by the following

\begin{prop} \label{prop10.2}
Let $(\Delta, L_{s\a t})$ and $(\Delta, K_{s\b t})$ be pseudo-canonic representations of languages $\L$ and $\K$ respectively and let $uv$ be a single split of a word $w$ such that $v$ is a valid expression of $\K$. 

If, additionally, $v$ is a non-ambiguous expression of $\K$ with the meaning $B_\b$ and $u$ an invalid expression of $\L$, then
\begin{equation} \label{eq10.3}
\Delta \vdash \t_{(s\a\b t)s(\a\x\b)} \: (L^+_{s\a t} * K^+_{s\b t}) \: /w/ = (\t_{(s\a t)s\a} \: L_{s\a t} \: /u/, \: B_\b).
\end{equation}

If, alternatively, $uv$ is a single split of a word $w$ such that $u$ is valid expression of $\L$ and, additionally, $u$ is a non-ambiguous expression of $\L$ with the meaning $A_\a$ and $v$ - an invalid expression of $\K$, then
\begin{equation} \label{eq10.4}
\Delta \vdash \t_{(s\a\b t)s(\a\x\b)} \: (L^+_{s\a t} * K^+_{s\b t}) \: /w/ = (A_\a, \: \t_{(s\b t)s\b} \: K_{s\b t} \: /v/).
\end{equation}
\end{prop}
\begin{proof}
In a similar way to the proof of Proposition \ref{prop7.7}, taking into account Proposition \ref{prop10.1}, for the first case we obtain
\[
\Delta \vdash (L^+_{s\a t} * K^+_{s\b t}) \: /w/ = (= \t_{(s\a t)s\a} \: L_{s\a t} \: /u/) \X (= B_\b)
\]
and for the second case
\[
\Delta \vdash (L^+_{s\a t} * K^+_{s\b t}) \: /w/ = (= A_\a) \X (= \t_{(s\b t)s\b} \: K_{s\b t} \: /v/).
\]
(\ref{eq10.3}) and (\ref{eq10.4}) follow from this according to (\ref{eq10.1}).
\end{proof}

The three features of the \emph{partial translations} (\ref{eq10.3}) and (\ref{eq10.4}) worth noting are:
\begin{enumerate}
\item they capture the correct overall structure of a precise meaning that the concatenation would acquire upon elaboration of an incomplete nested language;
\item they allow accurate restoration of their expressions $uv = w$ even while a nested language remains incomplete;
\item translation of only a single constituent is sufficient to convert a partial translation to the full translation $(A_\a, B_\b)$ upon an elaboration of an incomplete nested language.
\end{enumerate}

For example, consider the self-extension 
\[
S^+ \eqd \St \OR \l x_s (= \t_{(stt)st} \: \St \: x_s)
\]
of the sample grammar we built in section 4. If $\Delta_S$ denots the same set of axioms as in that section, then we will have, for example:
\begin{eqnarray*}
\Delta_S &\vdash& S^+ \; /\text{Jack paints a house}/ \\ 
         && \exists x_e (\t_{(s(eet)t)eet} \: \Verbt \: /\text{paint}/ \: x_e \: \Jack \;\AND\; \House \: x_e),\\
\Delta_S &\vdash& S^+ \; /\text{Jack builds a computer}/ \\ 
         && \exists x_e (\Build \: x_e \: \Jack \;\AND\; \t_{(s(et)t)et} \: \Noun \: /\text{computer}/ \: x_e),
\end{eqnarray*}
that illustrates both features 1 and 2.  Now assume lexicon $\Noun$ be extended with an entry $(= /\text{computer}/) \FA \Computer_{et}$ and let $\Delta'_S$ denote the changed axiom set. Then the derivation
\begin{eqnarray*}
\Delta'_S &\vdash& \exists x_e (\Build \: x_e \: \Jack \;\AND\; \t_{(s(et)t)et} \: \Noun \: /\text{computer}/ \: x_e) = \\
          && \exists x_e (\Build \: x_e \: \Jack \;\AND\; \Computer \: x_e)
\end{eqnarray*}
follows immediately from
\[
\Delta'_S \vdash \Noun \: /\text{computer}/ = (= \Computer),
\]
that illustrates the feature 3.

\section{Conclusion}

In this work we studied some classes of languages that can be defined entirely in terms of provability in an extension ($Ty_n^\A$) of sorted type theory by interpreting one of its base types as \emph{symbolic}, whose constants denote symbols of an alphabet $\A$ or their concatenations. The symbolic type ($s$) is quite similar to the phonological type of Higher Order Grammar (HOG) \cite{pollard:hana:2003,pollard:2004:cg,pollard:2006}. However, the theory $Ty_n^\A$ differs from the logic of HOG in that it does not contain special types for syntactic entities.

We launched from the two simple observations. 
First, given an arbitrary consistent set $\Delta$ of non-logical axioms and a term of type $s\a t$, one can define an $\a$-language, i.e. a relation between strings of symbols from alphabet $\A$ and terms of type $\a$, by the condition
\[
\Delta \vdash L_{s \a t} \: /w/ \: A_\a,
\]
where $/w/$ stands for an $s$-term denoting a word (string) $w \in \A^*$.
Secondly, any $Ty_n$\emph{-representable} language, i.e. a language that can be defined this way, necessarily possesses the semantic property of being \emph{logically closed}, meaning, loosely speaking, that the full semantic interpretation of every valid expression of the language is given by a whole class of logically equivalent terms.

Next, we proved that the inverse statement is also true, that is, every logically closed language is $Ty_n$-representable.  This general result, however, is not constructive in the sense that it does not provide an effective procedure of actually building a $Ty_n$ representation for a language defined by a finite set of rules or rule schemata. Our further efforts therefore concentrated mainly on finding such procedures for some important special classes of logically closed languages.

We began this with definitions of $\a$-lexicons and a special canonic representation and proved that a language has a canonic representation if and only if it is a lexicon. The canonic representation of a lexicon is efficiently found from the description of the lexicon.

Then we defined and investigated \emph{recursive grammars} - a kind of transformational grammars equipped with $Ty_n$ semantics. We introduced $Ty_n$ operators representing the basic language-construction operations used in the recursive grammar rules and proved that a compact $Ty_n$ representation for every language component of a recursive grammar comes from its rules in a quite straightforward manner. We illustrated this technique by simple examples, from which it became clear that language components of a recursive grammar stand for syntactic categories. We also briefly discussed classification of recursively defined $Ty_n$-representable languages in terms of Chomsky hierarchy and on this basis indicated some conditions of when such languages turn out to be decidable, in spite of the fact that $Ty_n$ is undecidable and hence so are generally $Ty_n$-representable languages.

Then we moved to a further specialization of $Ty_n^\A$ by introducing the \emph{context} type ($c$), which can be interpreted as a "store" or "state" and thus is similar to "state of affairs" or "World" types of \cite{gallin:1975} and \cite{pollard:2005:lacl}, respectively. We showed that our interpretation allows to model not only languages with intensional semantics (as in the above referenced works), but also so called \emph{instructive} and \emph{context dependent} languages. The latter come in our formalism in result of transformation of an $\a$-language to state-passing style (SPS) \cite{wadler:1992,sabry:1995,jones:2003} $c(c\x\a)$-language, which enables passing "side effects" of parsing some text constituents to others. An instructive, i.e. a $cc$-language can be considered as a context dependent transformation of a unit-type language and provides semantic features that principally could not be modeled by a regular $t$-language nor by its intensional ($ct$) transformation. We found $Ty_n$ representations of context-dependent language construction operations (anaphoric and cataphoric concatenations and context raising and instantiation) that allow generalization of the previous results for context-dependent languages. We demonstrated how this formalism can address pronoun anaphora resolution.

Finally, we defined \emph{translation} and \emph{expression} $Ty_n^\A$ operators and introduced a special language representation (\emph{self-extension}), revealing a useful property of \emph{partial translation}.

The proposed formalism efficiently addresses the known issues inherent in the two-component language models and also partially inherited by HOG. Indeed,

\begin{enumerate}

\item as there are no restrictions on a type of meanings of a $Ty_n$-representable language or a component of it, which type, in particular, can be constructed with use of the symbolic type too, such a language can generally express facts about itself or about another $Ty_n$-representable language;

\item the property of partial translation of a self-extensible $Ty_n$ language representation demonstrates its built-in lexical robustness;

\item as syntactic categories are not captured in the $Ty_n$ logic, but introduced only in non-logical axioms defining a concrete language, the $Ty_n$ representation provides full structural robustness and flexibility.

\end{enumerate}

These results have both theoretical and practical implications.
First, they facilitate modeling of such fundamental abilities associated with the human language acquisition process like communication of new grammar rules and lexical entries directly in an (already acquired) sub-set of the object language and independently acquiring new lexical entries upon encountering them in a known context, with postponed acquisition of their exact meanings.
Other possible applications of the proposed formalism may include:

\begin{enumerate}

\item addressing donkey pronouns \cite{Geach:Reference,Kamp:TSR,Heim:PhD,LLBook18955} and other dynamic semantic phenomena entirely within the mainstream higher-order logic semantics, as it was preliminarily outlined in section 8;

\item modeling advanced language self-learning mechanisms, for example, such like deriving or generalization of grammar rules from text samples (may be accompanied by their semantic interpretation, communicated by other means or just expressed differently in the same language), that might be reduced, in frame of this formalism, to a problem of higher-order unification \cite{Gardent97amulti-level,DM:04}.

\end{enumerate}
Finally, an implementation of the proposed formalism, like that suggested in \cite{patent_6999934}, should provide a platform for building NLP systems with the following powerful features:

\begin{enumerate}

\item the ability to recursively define more complex or more comprehensive languages in terms of previously defined simpler or limited ones;

\item automatically acquiring new lexical entries in the process of operation;

\item storing partial text parses which can be automatically completed later, upon extending the used language definition;

\item freely switching between different input and output languages to access the same semantic content.

\end{enumerate}

\end{document}